\documentclass[letterpaper]{article} 
\usepackage{aaai24} 
\usepackage{times}  
\usepackage{helvet}  
\usepackage{courier}  
\usepackage[hyphens]{url}  
\usepackage{graphicx} 
\usepackage{natbib}  
\usepackage{caption} 
\usepackage{algorithm}
\usepackage{algorithmic}

\usepackage{booktabs}
\usepackage{amsmath}
\usepackage{amsthm} 
\usepackage{amssymb} 
\usepackage{multirow} 
\usepackage{makecell}
\usepackage{pifont}

\newcommand \multiset[1] {\left\{\!\!\left\{#1\right\}\!\!\right\} }
\newcommand \spm[1] {\footnotesize{$\pm$#1} } 
\newcommand \f[1]{\textbf{#1}}
\newcommand \s[1]{\underline{#1}}
\newcommand \thi[1]{#1}

\theoremstyle{plain}
\newtheorem{theorem}{Theorem}[section]

\newtheorem{lemma}[theorem]{Lemma}

\theoremstyle{definition}
\newtheorem{definition}[theorem]{Definition}

\theoremstyle{remark}

\usepackage{newfloat}
\usepackage{listings}
\DeclareCaptionStyle{ruled}{labelfont=normalfont,labelsep=colon,strut=off} 
\floatstyle{ruled}
\newfloat{listing}{tb}{lst}{}
\floatname{listing}{Listing}

\title{Higher-Order Graph Convolutional Network with Flower-Petals Laplacians \\on Simplicial Complexes}

\author{
Yiming Huang\textsuperscript{\rm 1,2,}\equalcontrib,
Yujie Zeng\textsuperscript{\rm 1,2,}\equalcontrib,
Qiang Wu\textsuperscript{{\rm 3,}\textdagger},
Linyuan L{\"u}\textsuperscript{\rm 4,1,2,}\thanks{Corresponding author.}
}
\affiliations{
    \textsuperscript{\rm 1}Yangtze Delta Region Institute (Huzhou), University of Electronic Science and Technology of China\\
    \textsuperscript{\rm 2} Institute of Fundamental and Frontier Sciences, University of Electronic Science and Technology of China\\
    \textsuperscript{\rm 3} Research Institute of Electronic Science and Technology, University of Electronic Science and Technology of China\\
    \textsuperscript{\rm 4}School of Cyber Science and Technology, University of Science and Technology of China\\
    \{yiming\_huang, yujie\_zeng\}@std.uestc.edu.cn, \{qiang.wu, linyuan.lv\}@uestc.edu.cn
}

\begin{document}

\maketitle

\begin{abstract}
Despite the recent successes of vanilla Graph Neural Networks (GNNs) on various tasks, their foundation on pairwise networks inherently limits their capacity to discern latent higher-order interactions in complex systems.
To bridge this capability gap, we propose a novel approach exploiting the rich mathematical theory of simplicial complexes (SCs) - a robust tool for modeling higher-order interactions.
Current SC-based GNNs are burdened by high complexity and rigidity, and quantifying higher-order interaction strengths remains challenging.
Innovatively, we present a higher-order Flower-Petals (FP) model, incorporating FP Laplacians into SCs. Further, we introduce a Higher-order Graph Convolutional Network (HiGCN) grounded in FP Laplacians, capable of discerning intrinsic features across varying topological scales.
By employing learnable graph filters, a parameter group within each FP Laplacian domain, we can identify diverse patterns where the filters' weights serve as a quantifiable measure of higher-order interaction strengths.
The theoretical underpinnings of HiGCN's advanced expressiveness are rigorously demonstrated. Additionally, our empirical investigations reveal that the proposed model accomplishes state-of-the-art performance on a range of graph tasks and provides a scalable and flexible solution to explore higher-order interactions in graphs.
Codes and datasets are available at \url{https://github.com/Yiminghh/HiGCN}.
\end{abstract}

\section{Introduction}
Graphs are ubiquitous in representing irregular relations in various scenarios. 
However, they are inherently constrained to modeling pairwise interactions exclusively \cite{HigherOrderReview2020}.
Many empirical systems display group interactions, going beyond pairwise connections, such as social systems \cite{SocialNet2010}, neuronal networks \cite{BrainNet2011}, and ecological networks \cite{Ecology2017}.
However, such higher-order interactions can hardly be modeled or approximated by pairwise graphs.
In addition, it is still elusive how to quantify the higher-order interaction strength, although many studies have demonstrated its existence \cite{battiston2021physics}.

Graph neural networks (GNNs) can exploit the features and topology of graphs simultaneously, thereby triggering a wide-spreading research interest and endeavor in various graph learning tasks such as recommender systems and new drug discovery.
In particular, spectral GNNs have been widely recognized for their rigorous mathematical theory.
Nevertheless, pairwise-graph-based GNNs fail to capture latent higher-order interactions prevalent in empirical systems, and their expressive power was proved to be upper bounded by Weisfeiler-Lehman (WL) test \cite{GIN2019}.

Simplicial complexes (SCs) and hypergraphs have emerged to study higher-order interactions beyond conventional pairwise descriptors \cite{HigherOrderReview2020}. 
While hypergraph learning has made fruitful progress \cite{HypergraphLearning22review}, it typically ignores relations within the hyperedges, and the construction of hypergraphs is often under-optimized.
%
The simplicial description is another potent tool with elegant mathematical theories to draw from, paving a middle ground between graphs and hypergraphs.
It has been found that SCs play a vital role in social contagion, synchronization, brain network analysis, etc.  
%

Deep learning facilitated simplicial complex theory is a fresh perspective and a promising research field.
Several simplicial GNNs have been proposed by simply replacing the graph Laplacian with the Hodge Laplacian \cite{hodge_schaub2020random}. 
%
A simplicial WL test is proposed along with its neural version MPSN \cite{SWL2021} based on the adjacency relations that Hodge theory defines.
MPSN is proved to be more powerful than vanilla GNNs under ideal conditions, implying the potential of extending graph representation learning to SCs.
%

%
In summary, pairwise GNNs fail to capture latent group interactions prevalent in complex systems, and the expressive power of such models was proved to be upper bounded by the WL-test. 
As an emerging and promising research field, simplicial GNNs have initially shown their potential to outperform pairwise GNNs. 
However, existing models are limited by their high complexity and low flexibility.

In this paper, we introduce a novel higher-order flower-petals (FP) representation based on two-step random walk dynamics \cite{ISMnet} between the flower core and petals.
This representation enables us to incorporate interactions among simplices of various orders into graph learning.
Higher-order graph convolutional network (HiGCN) is then proposed by employing learnable and tailored convolution filters (group of parameters) in different FP Laplacian domains. 
The learnable filters can learn arbitrary shapes and deal with high and low-frequency parts of the simplicial signals adaptively. 
Hence, the proposed HiGCN model can learn the simplex patterns of disparate classes and higher-order structures simultaneously. 
Moreover, the filters' weights in different orders can quantify the higher-order interaction strength, contributing to a deeper understanding of higher-order mechanisms in complex systems.
We also interpret HiGCN from the message-passing perspective and theoretically demonstrate its superior expressive power. 
Numerical experiments on various graph tasks further pinpoint that the proposed model has outperformed state-of-the-art (SOTA) methods.

\paragraph{Main contributions.}

To summarise, we construct an innovative higher-order flower-petals (FP) model and FP Laplacians from the random walk dynamics to capture interactions among simplices of different orders. We then propose a higher-order graph convolutional network (HiGCN) based on our FP Laplacians, which is demonstrated to have superior expressiveness in theory and significant performance gains in various empirical experiments. Furthermore, a data-driven strategy is employed to quantify the higher-order interaction strength. In general, our work is an important step towards advancing higher-order graph learning and understanding higher-order mechanisms.

\section{Related Work}

In this section, we briefly review related works on vanilla spectral GNNs and higher-order GNNs.

\paragraph{Spectral GNNs.} 
Spectral GNNs are based on the graph Fourier transform \cite{graphFourier2013}, which employs the graph Laplacian eigenbasis as an analogy of the Fourier transform. 
ChebNet \cite{ChebNet} employs Chebyshev polynomials to replace the convolutional core, while GCN \cite{GCN} uses a first-order approximation of the convolution operator.
By considering the relationship between GCN and PageRank, APPNP \cite{APPNP} is proposed via personalized PageRank.
GPRGNN \cite{GPRGNN} leverages a learnable graph filter, exhibiting superiority in heterogeneous graph learning.
The filter forms of some spectral GNNs are summarized in Table \ref{tab:filter_form}.

\paragraph{Higher-order GNNs.} 

The crude simplification of complex interaction into pairwise will inevitably result in information loss. 
Higher-order GNNs, as extensions of vanilla GNNs, can be classified into different types according to their application scenarios, and spectral-based simplicial GNNs are in the limelight of this paper.
The Hodge theory \cite{Hodge_Hatcher} enables us to describe diffusion across simplices conveniently.
Several simplicial GNNs, such as SNN \cite{SNN2020} and SCoNe \cite{SCoNe21Roddenberry}, simply replace the graph Laplacian with the Hodge $p$-Laplacian. 
SCNN \cite{SCNN2022} employs flexible simplicial filters to process edge signals from lower and upper simplicial neighbors, respectively. 
BScNet \cite{BScNets} is introduced by replacing the graph Laplacian with the block Hodge Laplacian. 
Nevertheless, the Hodge theory is inherently constrained to modeling interactions between simplices within one order difference.
As for spatial models, SGAT \cite{SGAT} constructs SCs from heterogeneous graphs and leverages upper adjacencies to pass messages between simplices.
MPSN \cite{SWL2021} is designed based on the simplicial WL-test with four types of adjacency relations. 
Generally, most simplicial GNNs can only leverage information from specific simplicial orders, missing the inherent advantages of SCs.
Besides, it is computationally expensive to find all simplices \cite{maxClique1999} and unnecessary to compute embeddings for redundant simplices in traditional tasks.

\begin{table}[!t]
\centering
\renewcommand\tabcolsep{5pt} 
\resizebox{0.48\textwidth}{!}{
\begin{tabular}{p{4.426cm}ccc}
\toprule
{\bf Model}   
&{\bf Convolution Filter}                                     
&{\bf Spectral}                           
&{\bf Learnable}\\
\midrule
GCN \cite{GCN} &$(1-\lambda)^K$ &Graph Laplacian  &\ding{55}\\
APPNP \cite{APPNP} &$\sum_{k=0}^K\frac{\gamma^k}{1-\gamma} (1-\lambda)^k$ &Graph Laplacian  &\ding{55}\\
GPRGNN \cite{GPRGNN}  &$\sum_{k=0}^{K}\gamma_k(1-\lambda)^k$ &Graph Laplacian &$\gamma_{k}$\\
ChebNet \cite{ChebNet} &$\sum_{k=0}^K\gamma_k \cos{\left(k\arccos{\left(1-\lambda\right)}\right)}$ &Graph Laplacian &$\gamma_k, K$ \\
SNN \cite{SNN2020}  &$\lambda^K$  &Hodge Laplacian &\ding{55}\\
SCoNe \cite{SCoNe21Roddenberry} &$\lambda^K_{down}, \lambda^K_{up}$ &Hodge Laplacian &\ding{55} \\
SCNN \cite{SCNN2022} &$\sum_{k=0}^{K_1}\gamma_{d,k}\lambda_{down}^k + \sum_{k=0}^{K_2}\gamma_{u,k}\lambda_{up}^k $ & Hodge Laplacian &\ding{55} \\
BScNets \cite{BScNets} &$f(\lambda_1,\lambda_2,\cdots,\lambda_P;\theta)^K$    &\makecell{Block Hodge\\Laplacian} &$\theta$ \\
\midrule
\textbf{HiGCN} (Ours)   &$\sum_{k=0}^{K}\gamma_{p,k}(1-\lambda_p)^k,$ $p=1,2,\cdots,P$     &FP Laplacian &$\gamma_{p,k}$\\
\bottomrule
\end{tabular}}
\caption{The filter forms of spectral GNNs.}
\label{tab:filter_form}
\end{table}

\section{Preliminaries}

The background knowledge required to present this work better is illustrated in this section.
Let $\mathcal{G}=\left(\mathcal{V},\mathcal{E}\right)$ denote an undirected pairwise graph with a finite node set $\mathcal{V}=\left\{v_1,\cdots,v_n\right\}$ and an edge set $\mathcal{E} \subseteq \mathcal{V}\times\mathcal{V}$. Assume that $|\mathcal{V}|=n, |\mathcal{E}|=n_1$, and $N(v)$ denotes the set of nodes adjacent to node $v$ in $\mathcal{G}$, i.e., $N(v)=\left\{u \in \mathcal{V} | (v,u)\in \mathcal{E}\right\}$.
The nodes are associated with a node feature matrix $X \in \mathbb{R}^{n\times d}$, where $d$ signifies the number of features per node.

\begin{definition}[Simplicial complexes, SCs]
A simplicial complex $\mathcal{K}$ is a finite collection of node subsets closed under the operation of taking nonempty subsets, and such a node subset $\sigma \in \mathcal{K}$ is called a simplex (as illustrated in Figure~\ref{fig:SCs}). 
\end{definition}

\begin{figure}[!t]
\centering
\includegraphics[width=\linewidth]{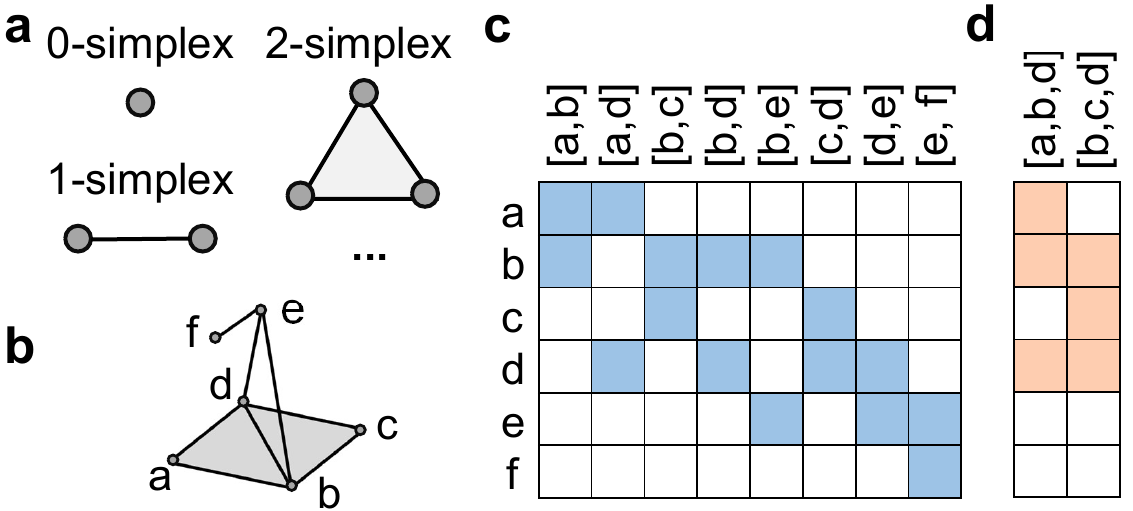}
\caption{a shows several typical simplices and its collection forms SCs in b. Subfigures c and d visualize the higher-order incidence matrices $\mathcal{H}_p$ for $p=1$ and $2$, respectively.}
\label{fig:SCs}
\end{figure}

SCs are a potent tool with a rich theoretical foundation upon algebraic and differential topology and geometry  \cite{HigherOrderReview2020}. 
Instead of predominantly studying pairwise interactions, SCs facilitate the modeling of higher-order interactions and multi-node graph structures.

A node subset $\sigma = \left[v_0,v_1,\cdots,v_p\right] \in \mathcal{K}$ with cardinality $p+1$ is referred to as a $p$-dimensional simplex, termed $p$-simplex, and we denote the set of all such $p$-simplices as $\mathcal{K}_p $ with $|\mathcal{K}_p| = n_p$.
One can regard vertices as $0$-simplices, edges as $1$-simplices, “filled” triangles as $2$-simplices, and so forth. 
A triangle $\left[v_1,v_2,v_3\right] \in \mathcal{K}$ implies that its nonempty subsets, namely $\left[v_1\right]$, $\left[v_2\right]$, $\left[v_3\right]$, $\left[v_1,v_2\right]$, $\left[v_1,v_3\right]$, and $\left[v_2,v_3\right]$, are also in $\mathcal{K}$.  
Pairwise graphs can be viewed as $1$-dimensional SCs, while higher-order SCs generally carry more structure information over pairwise graphs, which is critical and should not be omitted.

\textbf{Clique complex lifting transition}, as formally defined in Appendix B, extracts all cliques as simplices, converting pairwise graphs to SCs.
This transformation enables the study of pairwise graphs from simplicial perspectives.

The boundary relation describes which simplices lie on the boundary of other simplices. 
We say $\sigma$ is on the boundary of $\tau$, denoted as $\sigma \prec \tau$, iff $\sigma \subset \tau$ and $dim(\sigma) = dim(\tau)-1$.
%
For example, edges $\left[v_1,v_2\right]$, $\left[v_1,v_3\right]$, and $\left[v_2,v_3\right]$ lie on the boundary of the $2$-simplex $\left[v_1,v_2,v_3\right]$.

\textbf{Hasse diagram} is one of the most common representations of SCs, where each vertex corresponds to a simplex. 
The edges in the Hasse diagram are defined by the boundary relation, and there exists an edge connecting two vertices $\sigma_1$ and $\sigma_2$, iff $\sigma_1 \prec \sigma_2$.
The hasse diagram is highly expressive, and several simplicial GNNs \cite{SWL2021, GMPS22} are built precisely on the boundary relationships shown in the hasse diagram.

\section{Methodology}
We first introduce the higher-order flower-petals model for simplicial complex representation, which will subsequently be leveraged to construct our HiGCN model.

\subsection{Flower-Petals Model}
Hasse diagrams are valuable in studying SCs, but they are inherently constrained to modeling interactions for directly adjacent simplices and are computationally expensive to construct. 
Multiple transitions are required for information to pass between nodes and higher-order structures.
Besides, the number of total simplices grows exponentially with the number of nodes in dense graphs.
Computing embeddings for all higher-order structures can be costly and unnecessary for specific-level tasks. 
To address these challenges, we construct a novel higher-order representation, named the flower-petals (FP) model, and then introduce FP adjacency and Laplacian matrices based on the higher-order random walk dynamics between the flower core and petals.
%
 

It can be simplified only to consider the interaction between $0$-simplices and higher-order structures when tackling the most common tasks: node-level tasks.
Hence, we construct a flower-petals model by simplifying the intermediate vertices on the Hasse diagram. 
Specifically, the flower-petals model consists of one core and several petals, see Figure \ref{fig:fp_model}, with interactions considered only between the core and petals. 
$0$-simplices are placed in the flower core, and each flower petal involves simplices of the same order (larger than zero).
The term $p$-petal is used to represent the petal containing $p$-simplices.
Diverse and complex interactions exist between $p$-petal and the core, which can be unwrapped as a bipartite graph $\mathcal{G}_p$. 
Mathematically, the bipartite graph $\mathcal{G}_p$ consists of two distinct vertex sets $\left(\mathcal{V},\mathcal{K}_p\right)$, where $\mathcal{V}$ represents the set of nodes contained in the flower core and $\mathcal{K}_p$ comprised of simplices in the $p$-petal ($p \geq 1$). If simplex $\sigma (\in\mathcal{K}_p)$ contains node $v(\in\mathcal{V})$, then there exists an edge between their corresponding vertices in $\mathcal{G}_p$.
The proposed flower-petals model prunes the information interaction rules between petals but is still extremely expressive and useful. 

\begin{figure*}[!t]
\centering
\includegraphics[width=0.89\linewidth]{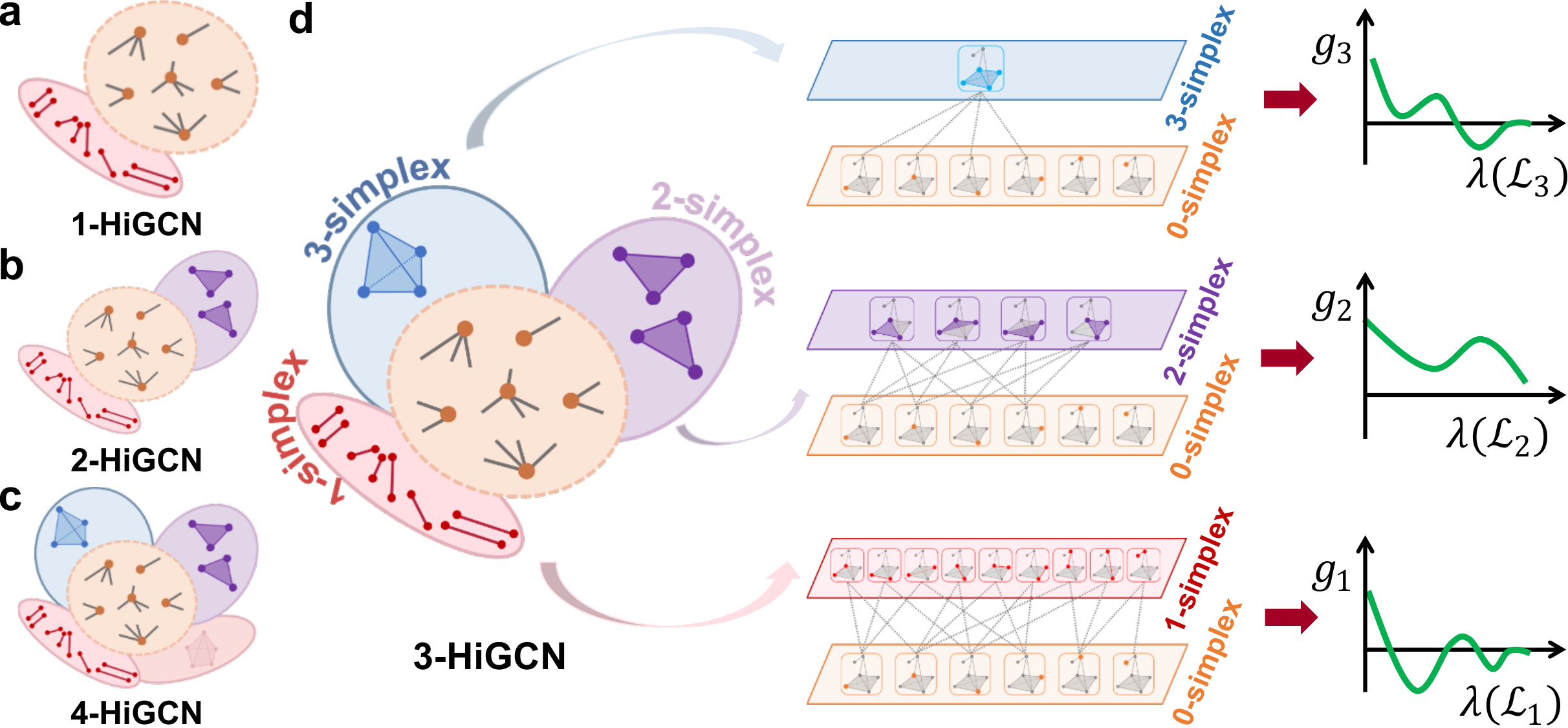}
\caption{Visualization of the ﬂower-petals model.
Different HiGCN models employ different numbers of petals, with each petal containing simplices of identical order. 
a-d visualizes 1,2,4,3-HiGCN, respectively.
The interaction between each petal and the flower core can be unwrapped as an individual bipartite graph. FP Laplacians are derived based on the random walk dynamics in the bipartite graphs, followed by various learnable convolution operations $g_p$ on each FP Laplacian basis.}
\label{fig:fp_model}
\end{figure*}

Inspired by incidence matrices in pairwise networks, we introduce higher-order incidence matrix $\mathcal{H}_p \in \mathbb{R}^{|\mathcal{V}| \times |\mathcal{K}_p| }$ to describe the association between vertices in the core and $p$-simplices in the $p$-petal, with entry $\mathcal{H}_p(v,\sigma)=1$ indicating the vertex $v$ is contained in the simplex $\sigma (\in \mathcal{K}_p)$. 
Visual representations are provided in Figure \ref{fig:SCs} \textbf{c} and \textbf{d} for clarity.

\subsection{Flower-Petals Algebraic Description} 

Hodge Laplacian \cite{hodge_schaub2020random, Hodge_Hatcher} is a fundamental tool in simplicial complexes.
However, it can only describe interactions between simplices within one-order differences. 
To model interactions between different order simplices more flexibly, we introduce novel matrical descriptions for simplicial complexes based on the random walk dynamics between the flower core and petals.

The main idea of random walks is to traverse a graph starting from a single node or a set of nodes and get sequences of locations \cite{HoRW}. 
We introduce the traditional random walk model in Appendix D.
Walking on the bipartite graphs $\mathcal{G}_p$ consists of two sub-steps: ($\mathrm{\uppercase\expandafter{\romannumeral 1}}$) upward walk and ($\mathrm{\uppercase\expandafter{\romannumeral 2}}$) downward walk.

The upward walk refers to the walk from nodes in the flower core to their corresponding simplices in the $p$-petal, while the downward walk proceeds in the opposite direction.
Consider $\pi(t) = \left(\pi_{v_1}(t),\cdots,\pi_{v_n}(t)\right)^\top$, whose item $\pi_\sigma(t)$ encodes the probability for simplex $\sigma$ to be occupied by a random walker at step $t$.
In the upward walk process, information is transmitted from nodes to simplices and the probability of moving from vertex $u$ to simplex $\sigma$ is equal to $\mathcal{H}_p(u,\sigma)/d_p(u)$. 
The downward walk, i.e., petals-to-core walk, allows information to be transferred from simplices back to nodes. These two processes follow that
$\pi_\sigma(t-1) = \sum_{u} {d_p(u)^{-1}}{\mathcal{H}_p(u,\sigma)}  \pi_u(t-2),$
and 
$\pi_v(t) = \sum_{\sigma} {\delta_p(\sigma)^{-1}} {\mathcal{H}_p(v, \sigma)} \pi_\sigma(t-1)$, 
for upward and downward walks, respectively.
Here, $d_p(u) = \sum_{\sigma \in \mathcal{K}_p} \mathcal{H}_p(u,\sigma)$ denotes the degree of $u$ on $\mathcal{G}_p$, and $\delta_p(\sigma) = \sum_{v \in \mathcal{V} }\mathcal{H}_p(v,\sigma) = p+1$ represents the degree of $p$-simplex $\sigma$ on  $\mathcal{G}_p (p \geq 1)$.

The two-step walk \cite{ISMnet} integrates both the upward and downward walks, allowing the information to be transmitted from the flower core and back through the petals. 
A complete two-step walk process follows that
\begin{equation}
    \pi_v(t) = \sum_{\sigma} \frac{\mathcal{H}_p(v,\sigma)}{\delta_p(\sigma)} \sum_{u} \frac{\mathcal{H}_p(u,\sigma)}{d_p(u)} \pi_u(t-2).
\label{equ:wander}
\end{equation}

We can further derive the matrix representation for the two-step walk as 
$\pi(t)=\mathcal{H}_p D_{p,h}^{-1} \mathcal{H}_p^\top D_{p,v}^{-1}\pi(t-2),$
where $D_{p,v}=\operatorname{diag}\left(d_p(v_1),\cdots,d_p(v_n)\right)$ and $D_{p,h}=\operatorname{diag}\left(\delta_p(\sigma_1),\cdots,\delta_p(\sigma_{|\mathcal{K}|})\right)$$=(p+1)I$.
%
By multiplying $D_{p,v}^{-1/2}$ from the left sides of this equation, we can obtain
\begin{equation}
    D_{p,v}^{-1/2}\pi(t) =\left[D_{p,v}^{-1/2} \mathcal{H}_p  D_{p,h}^{-1}   \mathcal{H}_p ^\top  D_{p,v}^{-1/2}\right]D_{p,v}^{-1/2} \pi(t-2).
\end{equation}

Therefore, based on the two-step walk dynamic between the flower core and petals, we can define higher-order flower-petals (FP) adjacency matrices analogously to the reduced adjacency matrices (see Appendix D) as
\begin{equation}
    \begin{split}
        \tilde{\mathcal{A}_p} 
     &= D_{p,v}^{-1/2}\mathcal{H}_p D_{p,h}^{-1} {\mathcal{H}_p}^\top D_{p,v}^{-1/2}\\
     &= \frac{1}{p+1}D_{p,v}^{-1/2}\mathcal{H}_p{\mathcal{H}_p}^\top D_{p,v}^{-1/2}. 
    \end{split}
\end{equation}

The Laplacian operator is crucial for the processing of relational data, and it bears resemblance to the Laplace-Beltrami operator in differential geometry.
On the basis of the FP adjacency matrices, we can likewise define a series of higher-order FP Laplacian operators as
$
 \mathcal{L}_p = I - \tilde{\mathcal{A}_p}.
$

\begin{theorem}
\label{theorem: semi-positive}
The flower-petals adjacency matrices $\tilde{\mathcal{A}}_p$ and flower-petals Laplacian matrices $\mathcal{L}_p$ are all symmetric positive semidefinite.
\end{theorem}

It follows from Theorem \ref{theorem: semi-positive} that $0 \leq \lambda(\tilde{\mathcal{A}}_p), \lambda\left(\mathcal{L}_p\right)\leq 1$.
We defer the proof and further theoretical analysis of the spectral properties to Appendix A.
Theorem \ref{theorem: semi-positive} contributes to alleviate the numerical instability and exploding/vanishing gradients that may arise in the implementation of deep GNNs based on the FP Laplacians.
The diverse FP Laplacian matrices capture the various connectivity relations of the simplicial complexes, where we can learn a series of diverse spectral convolution operators.

\subsection{Higher-Order Graph Convolutional Network}

The eigen decomposition $\mathcal{L} =\Phi \Lambda \Phi^{\top}$ can be applied to the Laplacian matrix to obtain orthonormal eigenvectors $\Phi=\left(\phi_{1}, \phi_{2} \cdots, \phi_{n}\right)$ and a diagonal matrix $\Lambda = \operatorname{diag}\left(\lambda_{1},\lambda_{2},\cdots,\lambda_{n}\right)$. 
Then, for a graph signal $x$, the graph Fourier transform is defined as $\Phi^{\top} x$, where the eigenvectors act as the Fourier bases and the eigenvalues are interpreted as frequencies. The spectral convolution of signal $x$ and filter $g$ can then be formulated as
\begin{equation}
    g \star x =\Phi \left(\left( \Phi^{\top} g\right) \odot \left(\Phi^{\top} x\right)\right) =\Phi g(\Lambda) \Phi^{\top} x.
\end{equation}
Here, operator $\odot$ presents the Hadamard product, and the filter $g(\Lambda)$ applies $g$ element-wisely to the diagonal entries of $\Lambda$, i.e., $g(\Lambda)= \operatorname{diag}$ $  \left(g(\lambda_1), \cdots, g(\lambda_n)\right)$.
Note that spectral decomposition for large-scale networks can be computationally expensive. Therefore, one can approximate any graph filter using a polynomial filter with enough terms \cite{graphFourier2013}. 
Consequently, the filter $g$ is usually set to be a truncated polynomial $g(\lambda)=\sum_{k=0}^K \gamma_k \lambda^k$ of order $K$.
In this way, spectral decomposition is avoided.

We derive various FP Laplacian matrices based on the FP model, each representing different connectivity relations within SCs. 
Subsequently, we define different convolution operations on each FP Laplacian basis as
\begin{equation}
    g \star_p x = g_p(\mathcal{L}_p)x,
\end{equation}
where graph filter $g_p(\mathcal{L}_p)=\sum_{k=0}^K \gamma_{p,k} \mathcal{L}_p^k$ is composed of different learnable polynomial functions in each FP spectral domain. 
These learnable coefficients $\gamma_{p,k}$ capture the contributions of different hop neighbors in each order. 
$K$ is a hyperparameter denoting the largest hops of the simplices under consideration. 
%
We summarize some prevalent filter forms in Table \ref{tab:filter_form}.
%
When processing a simplicial signal $X\in\mathbb{R}^{n\times d}$ with $d$ dimensional features, a more general form of spectral GNNs follows that 
$Y = \rho \left( g(\mathcal{L})\varphi(X)\right)$.
Here, $\rho$ and $\varphi$ are permutation-invariant functions. 

To encode multi-scale higher-order information, the final prediction is obtained by concatenating results from different convolution operations as
\begin{equation}
    Y_p = g_p(\mathcal{L}_p)\varphi_p(X),
    \quad
    Y = \rho \left( {\big|\!\big|}_{p=1}^P Y_p\right).
\label{equ:general_HiSCN}
\end{equation}
Here, $\|$ concatenates the representation in different spectral domains. 
Besides, we simplify $\rho$ and $\varphi_p$ to linear functions as suggested by \cite{LinearGNN}, resulting in
\begin{equation}
    Y =   \mathop{\big|\!\big|}\limits_{p=1}^P \left(\sum_{k=0}^K \gamma_{p,k} \tilde{\mathcal{A}_p^k} X \Theta_p \right) W.
\end{equation}
Here, $\gamma_{p,k}, \Theta_p $, and $W$ are trainable parameters, and $P$ is a hyperparameter indicating the highest order of the simplices under consideration. 
The model under $P=\ell$ is denoted as $\ell$-HiGCN.
%
%
%
%
%
Notably, the training process can be accelerated by precalculating $\tilde{\mathcal{A}_p^k}$, which can be efficiently calculated between sparse matrices.

HiGCN facilitates the independent and flexible learning of filter shapes across disparate FP spectral domains rather than predetermining filter configurations. 
Consequently, it is adept at handling both high/low frequency and high/low order signal components in a versatile manner. 
Furthermore, we find that the filters' weights in different orders quantify the strength of the higher-order interactions, contributing to the understanding of higher-order mechanisms inherent within complex systems.
Now, we proceed to elucidate the advantages of HiGCN from various perspectives.


\paragraph{Expressive power.}

We have developed the HiGCN model from a spectral perspective. 
The WL test provides a well-studied framework for unique node labeling, and an intrinsic theoretical connection has been uncovered between the WL test and message-passing-based GNNs \cite{GIN2019}. 
We extend this relation and propose a higher-order WL test, termed HWL, along with its simplified version SHWL. 
Detailed procedures for WL, HWL, and SHWL are elaborated in Appendix B.
Furthermore, we revisit the HiGCN model from the message-passing perspective in Appendix  B, offering an alternative interpretation that underscores the exceptional expressive power of our model.
%

%

\begin{theorem}
\label{the:SHWL}
SHWL with clique complex lifting is strictly more powerful than Weisfeiler-Lehman (WL) test.
\end{theorem}



The proposed model can be interpreted as a neural version of the SHWL test where colors are replaced by continuous feature vectors.
Hence, Theorem \ref{the:SHWL} implies that HiGCN endows with greater expressive power than vanilla GNNs. 
See Appendix B for proof and detailed discussion.

\paragraph{Relation to existing models.}
HiGCN shows superiority over pairwise graph-based GCNs for exploiting higher-order information, and it generalizes spectral convolution operations on pairwise graphs, including GCN \cite{GCN} and GPRGNN \cite{GPRGNN}.
On the other hand, HiGCN exhibits greater flexibility than certain Hodge Laplacian-based simplicial GCNs, such as SNN \cite{SNN2020} and SCNN \cite{SCNN2022}, overcoming the constraints of information exchange exclusively through boundary operators.
Further derivation and discussion are presented in Appendix C.


\paragraph{Symmetries.}
It is a fundamental concept for understanding GNNs and their behavior.
HiGCN has been demonstrated to exhibit equivariance with respect to relabeling of simplices, enabling it to exploit symmetries in SCs. 
Formal proofs and detailed discussions are deferred to Appendix E.

\paragraph{Computational complexity.}
A balance between performance and complexity can be achieved by limiting the number of petals $P$. 
We find that a small $P$ is typically adequate, and considering more petals may result in diminishing marginal utility.
Generally, the computational complexity of HiGCN is comparable to that of spectral GNNs. 
We report the average training time per epoch and average total running time in Appendix G, demonstrating that HiGCN achieves competitive performance with a reasonable computational cost.
Additionally, when the targeted graph is not in the form of SCs, one should also consider the one-time preprocessing procedure for graph lifting, see Appendix G for details. 
\section{Experiments}
In this section, we evaluate HiGCN on three tasks: node/ graph classification and simplicial data imputation.
Detailed data introduction and experimental settings are deferred to Appendices H and I, respectively.

\subsection{Node Classification on Empirical Datasets}

\begin{table*}[!ht] 
\renewcommand\tabcolsep{0.9pt} 
\resizebox{\textwidth}{!}{
\begin{tabular}{cccccccccccc}
\toprule
Method    &Cora   &Citeseer   &PubMed &Computers  &Photo  &Chameleon  &Actor &Squirrel   &Texas  &Wisconsin\\
\midrule
MLP
&76.96\spm{0.95}  &76.58\spm{0.88}  &85.94\spm{0.22}  &82.85\spm{0.38}  &84.72\spm{0.34}  &46.85\spm{1.51}  &40.19\spm{0.56}  &31.03\spm{1.18}  &91.45\spm{1.14}  &93.56\spm{0.87}\\
GAT      
&88.03\spm{0.79}  &80.52\spm{0.71}  &87.04\spm{0.24}  &83.33\spm{0.38}  &90.94\spm{0.68}  &63.90\spm{0.46}  &35.98\spm{0.23}  &42.72\spm{0.33}  &78.87\spm{0.86}  &65.64\spm{1.74}\\
ChebNet  
&86.67\spm{0.82}  &79.11\spm{0.75}  &87.95\spm{0.28}  &87.54\spm{0.43}  &93.77\spm{0.32}  &59.96\spm{0.51}  &38.02\spm{0.23}  &40.67\spm{0.31}  &86.08\spm{0.96}  &90.57\spm{0.91}\\
BernNet  
&88.52\spm{0.95}  &80.09\spm{0.79}  &88.48\spm{0.41}  &87.64\spm{0.44}  &93.63\spm{0.35}  &\s{68.29}\spm{1.58}  &\s{41.79}\spm{1.91}  &\s{51.35}\spm{0.73}  &\f{93.12}\spm{0.65} &91.82\spm{0.38}\\
GGCN  
&87.68\spm{1.26}  &77.08\spm{1.32}  &89.63\spm{0.46}  &N/A  &89.92\spm{0.97}  &62.72\spm{2.05}  &38.09\spm{0.88}  &49.86\spm{1.55}  &85.81\spm{1.72} &87.65\spm{1.50}\\
APPNP  
&88.14\spm{0.73}  &80.47\spm{0.74}  &88.12\spm{0.31}  &85.32\spm{0.37}  &88.51\spm{0.31}  &51.89\spm{1.82}  &39.66\spm{0.55}  &34.71\spm{0.57}  &90.98\spm{1.64}  &64.59\spm{0.97}\\
GPRGNN 
&88.57\spm{0.69}  &80.12\spm{0.83}  &88.46\spm{0.33}  &86.85\spm{0.25}  &93.85\spm{0.28}  &\thi{67.28}\spm{1.09}  &39.92\spm{0.67}  &50.15\spm{1.92}  &\s{92.95}\spm{1.31}  &88.54\spm{1.37}\\
\midrule
k-S2V 
&68.30\spm{0.35} & 44.22\spm{0.44}  &67.21\spm{0.41} &84.15\spm{0.11} &89.08\spm{0.65}  
&49.00\spm{0.82} &N/A &39.15\spm{0.49}  &85.12\spm{0.98}  &87.44\spm{0.77} \\  
S2V
&80.15\spm{0.88}  &78.21\spm{0.34}  &85.48\spm{0.33} &83.25\spm{0.50} &84.33\spm{0.19} &47.14\spm{0.32}  &39.22\spm{0.50}  &40.26\spm{0.74}  &82.12\spm{0.23} &83.48\spm{0.89}\\
SNN     
&87.13\spm{1.02}  &79.87\spm{0.68}  &86.73\spm{0.28}  &83.33\spm{0.32}  &88.27\spm{0.74}  &60.96\spm{0.78}  &30.59\spm{0.23}  &45.66\spm{0.39}  &75.16\spm{0.96}  &61.93\spm{0.83}\\
SGAT 
&77.49\spm{0.79}  &78.93\spm{0.63} &88.10\spm{0.59} &N/A           &N/A               &51.23\spm{0.36}  &36.71\spm{0.49}  &N/A            &89.83\spm{0.66}  &81.47\spm{0.64}\\
SGATEF 
&78.12\spm{0.85}  &79.16\spm{0.72} &88.47\spm{0.62} &N/A           &N/A               &51.61\spm{0.40}  &37.33\spm{0.58}  &N/A            &89.67\spm{0.74}  &81.59\spm{0.81}\\
1-HiGCN
&\thi{88.96}\spm{0.28}  &\s{80.96}\spm{0.27}  &\thi{89.83}\spm{0.73}  &\thi{90.50}\spm{0.52}  &\s{95.22}\spm{0.30} &63.55\spm{0.84}  &\thi{41.57}\spm{0.27}  &49.13\spm{0.33}  &90.36\spm{0.78} &\thi{94.39}\spm{0.94}\\
2-HiGCN
&\f{89.23}\spm{0.23}  &\f{81.12}\spm{0.28}  &\s{89.89}\spm{0.16}  &\f{90.76}\spm{0.27}  &\f{95.33}\spm{0.37}  &\f{68.47}\spm{0.45}  &\f{41.81}\spm{0.52}  &\f{51.86}\spm{0.42}  &\thi{92.15}\spm{0.73}      &\s{94.69}\spm{0.95}\\
3-HiGCN
&\s{89.00}\spm{0.26}  &\thi{80.90}\spm{0.22}  &89.73\spm{0.17} &\s{90.65}\spm{0.20} &\thi{94.40}\spm{0.31} &67.12\spm{0.32} &41.29\spm{0.20} &\thi{50.92}\spm{0.34}  &91.85\spm{0.62}          &94.12\spm{0.68}\\
4-HiGCN
&88.63\spm{0.28}  &80.47\spm{0.31}  &\f{89.95}\spm{0.13} &90.35\spm{0.31} &94.10\spm{0.24} &66.98\spm{0.23} &41.13\spm{0.24} &50.45\spm{0.21}  &91.42\spm{0.75}  &\f{94.89}\spm{0.65}\\
\bottomrule
\end{tabular}}
\caption{Node classification results on empirical benchmark networks: mean accuracy $(\%)\pm 95\%$ confidence interval. The best results are in bold, while the second-best ones are
underlined.}
\label{tab:node_classify}
\end{table*}

We perform the node classification task employing five homogeneous graphs, encompassing three citation graphs - Cora, CiteSeer, PubMed \cite{Yang2016CiteGraph} - and two Amazon co-purchase graphs, Computers and Photo \cite{Shchur2018}. 
Additionally, we include five heterogeneous graphs, namely Wikipedia graphs Chameleon and Squirrel \cite{Rozemberczki2021}, the Actor co-occurrence graph, and the webpage graphs Texas and Wisconsin from WebKB \cite{Pei2020Geom-GCN}.
Adjacent nodes in homogeneous graphs tend to share the same label, while the opposite holds in heterogeneous graphs. 
The clique complex lifting transition is carried out on each graph.

We compare HiGCN with various baseline models including MLP, pairwise GNNs (GAT \cite{GAT2018}, ChebNet \cite{ChebNet}, BernNet \cite{BernNet}, GGCN \cite{GGCN}, APPNP \cite{APPNP}, GPRGNN \cite{GPRGNN}), and higher-order models (S2V \cite{Simplex2vec}, k-S2V \cite{k-simplex2vec}, SNN \cite{SNN2020}, SGAT, SGATEF \cite{SGAT}). 
We randomly partition the node set into train/validation/test subsets with a ratio of 60\%/20\%/20\%, and repeat the experiments 100 times. 
The mean classification accuracies on the test nodes are reported in Table \ref{tab:node_classify}.

It can be drawn from Table \ref{tab:node_classify} that HiGCN achieves the best results in 9 out of the 10 graphs. 
On the remaining dataset, HiGCN also displays comparable performance to the SOTA methods.
Generally, 2-HiGCN and 3-HiGCN outperform 1-HiGCN, suggesting the value of higher-order information in graph learning.
However, it is elusive to find that performance did not consistently increase with the inclusion of more higher-order interactions.
One possible explanation is that introducing more higher-order interactions might make the training process more complex and challenging. 
If the model lacks sufficient training data or appropriate training strategies, it may struggle to effectively harness these higher-order interactions.
Furthermore, HiGCN shows on average a greater lead on homogeneous graphs, consistent with the intuition that higher-order effects tend to manifest on homogeneous graphs \cite{HigherOrderReview2020}.
%

In addition, we scale to three larger datasets: Ogbn-arxiv and Genius (homogeneous graphs) and Penn94 (heterogeneous graph). 
The results in Appendix G highlight HiGCN's superior performance and robust scalability.

\textbf{Quantifying higher-order strength.}
The filter weights $\gamma_{p,k}$ captures the influence of $p$-simplex on $k$-hop neighbors; thus, we quantify the $p$-order interaction strength in terms of 
\begin{equation}
  \mathcal{S}_p = \sum\nolimits_k |\gamma_{p,k}|.  
\end{equation}
To gain insight, we visualize $\mathcal{S}_p$ with order $p=1,2,3,4$ on both homogeneous (Cora, Photo) and heterogeneous (Actor, Texas) graphs in Figure \ref{fig:stack5} \textbf{a}-\textbf{d}.
We observe that $\mathcal{S}_p$ decreases gently with the increase of $p$ in homogeneous graphs, while it decreases rapidly in heterogeneous graphs. 
This observation implies that the strength of higher-order effects varies at different orders and across different types of graphs. 

\begin{figure}[!t]
\centering  
\includegraphics[width=0.86\linewidth]{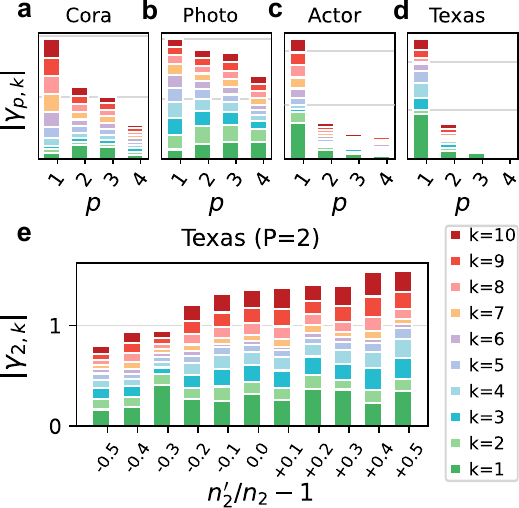}
\caption{a, b, c and d visualize the stack of learned weights $|\gamma_{p,k}|$ under order $p=1,2,3,4 (P=4)$. e visualizes the stack of $|\gamma_{2,k}|$ for Texas under various relative densities $\rho_2$.}
\label{fig:stack5}
\end{figure}

%
We observe that graphs with fewer higher-order structures tend to exhibit a smaller $\mathcal{S}_p$, potentially degrading HiGCN's performance.
For instance, in Texas, the only dataset where HiGCN's performance is not optimal, we note significantly weaker higher-order interactions compared to lower-order ones (see Figure \ref{fig:stack5}\textbf{d}), and it has the fewest triangles among all datasets (see Table 8).
To verify this conjecture, we manipulate the number of higher-order structures by adjusting the edge connectivity while maintaining the degree distribution as done in 1k null models \cite{zeng2023hyper}, see Appendix F for details. 
We define the relative higher-order density for the modified networks as
$\rho_p = {n_p^\prime}/{n_p}-1$,
where $n_p$ and $n_p^\prime$ denote the number of $p$-simplex in the original and the modified network, respectively.
Figure \ref{fig:stack5}\textbf{e} visualizes $\mathcal{S}_p$ under different $\rho_2$ for Texas, showing an upward trend as the triangle density $\rho_2$ increases.
Table 5 also reveals an increasing accuracy rank of HiGCN with the rise of $\rho_2$.
Hence, $\mathcal{S}_p$ can serve as a  quantification of $p$-order interaction strength. More results and discussions are deferred to Appendix F.

\subsection{Graph Classification on TUD Benchmarks}
To verify the broad applicability of the proposed model, we also evaluate the graph classification performance of HiGCN using various datasets from diverse domains, which are categorized into two main groups: bioinformatics datasets (i.e., PROTEINS \cite{data:proteins2005}, MUTAG \cite{data:MUTAG}, PTC \cite{data:PTC}) and social network datasets (i.e., IMDB-B, IMDB-M \cite{data:IMDB-REDDIT}).
To obtain a global embedding for each graph, we apply readout operations by performing averaging or summation. 
%
Following the standard pipeline in \cite{GIN2019}, we conduct a 10-fold cross-validation procedure and report the maximum average validation accuracy across folds.
The performance of HiGCN is presented in Table \ref{tab:TUD}, alongside the results for kernel methods (RWK \cite{RWK2003}, GK \cite{GK_2009efficient}, PK \cite{PK_2016neumann}, WL kernel \cite{WL_kernel_2011shervashidze}), pairwise GNNs (DCNN \cite{DCNN_atwood2016}, DGCNN \cite{DGCNN_zhang2018}, IGN \cite{IGN_2018}, GIN \cite{GIN2019}, PPGNs \cite{PPGNs_maron2019}, Natural GN \cite{natural_GN_2020}), and the higher-order model MPSN \cite{SWL2021}.

Our model exhibits superior performance compared to these baselines, demonstrating strong empirical results across all benchmark datasets. 
Additionally, HiGCN achieves its optimal outcomes on the two social network datasets, coinciding with the finding that simplices play a pivotal role in social networks \cite{battiston2021physics}.

\begin{table}[!t]
\centering
\renewcommand\tabcolsep{0.8pt} 
\resizebox{0.48\textwidth}{!}{
\begin{tabular}{cccccc}
\toprule
Dataset    & PROTEINS     & MUTAG                 & PTC               & IMDB-B               & IMDB-M        \\ \midrule
RWK        & 59.6\spm{0.1} & 79.2\spm{2.1}        & 55.9\spm{0.3}      & N/A                 & N/A           \\
GK (k=3)   & 71.4\spm{0.3} & 81.4\spm{1.7}        & 55.7\spm{0.5}      & N/A                 & N/A           \\
PK         & 73.7\spm{0.7} & 76.0\spm{2.7}        & 59.5\spm{2.4}      & N/A                 & N/A           \\
WL kernel  & 75.0\spm{3.1} & 90.4\spm{5.7}  & 59.9\spm{4.3}      & 73.8\spm{3.9}       & 50.9\spm{3.8}  \\
DCNN       & 61.3\spm{1.6} & N/A                  & N/A                & 49.1\spm{1.4}       & 33.5\spm{1.4}  \\
DGCNN      & 75.5\spm{0.9} & 85.8\spm{1.8}        & 58.6\spm{2.5}      & 70.0\spm{10.9}      & 47.8\spm{10.9} \\
IGN        & 76.6\spm{5.5} & 83.9\spm{13.0}       & 58.5\spm{6.9}      & 72.0\spm{5.5}       & 48.7\spm{3.4}  \\
GIN        & 76.2\spm{2.8} & 89.4\spm{5.6}        & 64.6\spm{7.0}      & \thi{75.1}\spm{5.1} & \thi{52.3}\spm{2.8}  \\
PPGNs      & \f{77.2}\spm{4.7} & \s{90.6}\spm{8.7}& \s{66.2}\spm{6.6}  & 73.0\spm{15.8}      & 50.5\spm{3.6}  \\
Natural GN & 71.7\spm{1.0} & 89.4\spm{1.6}        & \f{66.8}\spm{1.7}  & 73.5\spm{2.0}       & 51.3\spm{1.5}  \\
MPSN       & \thi{76.7}\spm{4.6} & \thi{89.8}\spm{5.5}  & 61.8\spm{9.1}  & \s{75.6}\spm{3.2}   & \s{52.4}\spm{2.9}  \\
\textbf{HiGCN}      
           &\s{77.0}\spm{4.2} & \f{91.3}\spm{6.4} & \s{66.2}\spm{6.9} & \f{76.2}\spm{5.1} & \f{52.7}\spm{3.5}  \\ \bottomrule
\end{tabular}}
\caption{Graph classification results. The best results are in bold, while the second-best ones are underlined.}
\label{tab:TUD}
\end{table}

\subsection{Simplicial Data Imputation}
In the previous two experiments, we focused on pairwise graphs with clique complex lifting. Now, we extend our investigation to impute missing signals in coauthorship complexes, a typical SC, wherein a paper with $p+1$ authors is represented by a $p$-simplex, and the $p$-simplicial signal corresponds to the number of collaborative publications among authors in the $p$-simplex. 
We employ three coauthorship complexes, namely DBLP \cite{data:DBLP-Benson2018}, History and Geology \cite{data:MAG-His-Geo}.  
The known signals for $0$-simplex are set to range from 10\% to 70\% (in units of 20\%), and the remainders are regarded as missing signals, replaced by the median of known signals. 
We apply Kendall's Tau $\mathcal{T}$ to measure the correlation between true and predicted simplicial signal, with $\mathcal{T}$ approaching 1 indicating superior performance \cite{kendall1938new}.
The experiment is repeated for 10 different random weight initializations, and the results are compared against higher-order models (namely SNN, SGAT, and SGATEF).

Table \ref{tab:simplicial_imputation} shows that HiGCN outperforms other higher-order benchmarks.
This superiority is mainly due to the inherent flexibility of our model in capturing higher-order information, whereas the benchmarks are restricted to learning through upper or lower adjacencies.
%
%
Moreover, HiGCN achieves more performance gains when less information is available. This may be attributed to higher-order information compensating for missing signals, 
with potential overlap when there is an abundance of known information.

\begin{table}[!t]
\centering
\renewcommand\tabcolsep{0.9pt} 
\resizebox{0.48\textwidth}{!}{
\begin{tabular}{cccccc}
\toprule
SCs  & Method & 10\% & 30\%  & 50\%  & 70\%   \\\midrule
\multirow{4}{*}{\rotatebox{90}{History}}
& SNN            & 0.201\spm{0.013} & 0.354\spm{0.016} & 0.495\spm{0.002} & 0.661\spm{0.002} \\                    
& SGAT           & 0.180\spm{0.010}    & 0.330\spm{0.002} & 0.432\spm{0.016} & 0.602\spm{0.005} \\                 
& SGATEF         & 0.200\spm{0.002}    & 0.340\spm{0.017} & 0.454\spm{0.021} & 0.633\spm{0.012} \\                     
& \textbf{HiGCN} & \f{0.258}\spm{0.004} & \f{0.438}\spm{0.002} & \f{0.579}\spm{0.005} & \f{0.666}\spm{0.009} \\\midrule
\multirow{4}{*}{\rotatebox{90}{Geology}}                    
& SNN            & 0.265\spm{0.022} & 0.417\spm{0.004} & 0.594\spm{0.02}  & 0.704\spm{0.003} \\
& SGAT           & 0.223\spm{0.004}    & 0.345\spm{0.030} & 0.599\spm{0.009} & 0.631\spm{0.008} \\
& SGATEF         & 0.230\spm{0.002}    &0.369\spm{0.018} & 0.615\spm{0.031} & 0.682\spm{0.012} \\
& \textbf{HiGCN} & \f{0.463}\spm{0.012} & \f{0.565}\spm{0.007} & \f{0.644}\spm{0.014} & \f{0.708}\spm{0.002} \\\midrule
\multirow{4}{*}{\rotatebox{90}{DBLP}}                      
& SNN            & 0.222\spm{0.021} & 0.348\spm{0.008} & 0.496\spm{0.005} & 0.668\spm{0.003}\\
& SGAT           & 0.210\spm{0.015}    & 0.279\spm{0.054} & 0.487\spm{0.022} & 0.643\spm{0.017}\\
& SGATEF         & 0.223\spm{0.004}    & 0.311\spm{0.002} & 0.491\spm{0.008} & 0.678\spm{0.005}\\
& \textbf{HiGCN} & \f{0.385}\spm{0.011} & \f{0.511}\spm{0.004} & \f{0.587}\spm{0.021} & \f{0.685}\spm{0.002}\\\bottomrule
\end{tabular}}
\caption{Simplicial data imputation results: mean Kendall correlation$\pm$standard deviation. The best results are in bold.}
\label{tab:simplicial_imputation}
\end{table}

\section{Conclusion}

This paper introduces a novel higher-order representation, the flower-petals (FP) model, enabling interactions among simplices of arbitrary orders.
To increase efficiency, we simplify the interaction rules in SCs.
It is a valuable and open question whether other simplifications would be more effective for specific tasks.
FP adjacency and Laplacian matrices are further introduced based on the higher-order random walk dynamics on the FP model. 
As an application of FP Laplacians in deep learning, a higher-order graph convolutional network (HiGCN) is introduced.  
Our theoretical analysis highlights HiGCN's advanced expressiveness, supported by empirical performance gains across various tasks.
Moreover, we deploy a data-driven strategy to demonstrate the existence of higher-order interactions and quantify their strength.
This work promises to offer novel insights and serve as a potent tool in higher-order network analysis.

\section*{Acknowledgments}
The authors acknowledge the STI 2030—Major Projects (Grant No. 2022
ZD0211400), the National Natural Science Foundation of China (Grant No.T2293771), the Sichuan Science and Technology Program (Grant
No.2023NSFSC1919) and the New Cornerstone Science Foundation through the XPLORER PRIZE.

\bibliography{ref.bib}

\begin{thebibliography}{68}
\providecommand{\natexlab}[1]{#1}

\bibitem[{Atwood and Towsley(2016)}]{DCNN_atwood2016}
Atwood, J.; and Towsley, D. 2016.
\newblock Diffusion-convolutional neural networks.
\newblock \emph{Advances in neural information processing systems}, 29.

\bibitem[{Battiston et~al.(2021)Battiston, Amico, Barrat, Bianconi, Ferraz~de Arruda, Franceschiello, Iacopini, K{\'e}fi, Latora, Moreno et~al.}]{battiston2021physics}
Battiston, F.; Amico, E.; Barrat, A.; Bianconi, G.; Ferraz~de Arruda, G.; Franceschiello, B.; Iacopini, I.; K{\'e}fi, S.; Latora, V.; Moreno, Y.; et~al. 2021.
\newblock The physics of higher-order interactions in complex systems.
\newblock \emph{Nature Physics}, 17(10): 1093--1098.

\bibitem[{Battiston et~al.(2020)Battiston, Cencetti, Iacopini, Latora, Lucas, Patania, Young, and Petri}]{HigherOrderReview2020}
Battiston, F.; Cencetti, G.; Iacopini, I.; Latora, V.; Lucas, M.; Patania, A.; Young, J.-G.; and Petri, G. 2020.
\newblock Networks beyond pairwise interactions: Structure and dynamics.
\newblock \emph{Physics Reports}, 874: 1--92.

\bibitem[{Benson et~al.(2018)Benson, Abebe, Schaub, Jadbabaie, and Kleinberg}]{data:DBLP-Benson2018}
Benson, A.~R.; Abebe, R.; Schaub, M.~T.; Jadbabaie, A.; and Kleinberg, J. 2018.
\newblock Simplicial closure and higher-order link prediction.
\newblock \emph{Proceedings of the National Academy of Sciences}.

\bibitem[{Billings et~al.(2019)Billings, Hu, Lerda, Medvedev, Mottes, Onicas, Santoro, and Petri}]{Simplex2vec}
Billings, J. C.~W.; Hu, M.; Lerda, G.; Medvedev, A.~N.; Mottes, F.; Onicas, A.; Santoro, A.; and Petri, G. 2019.
\newblock Simplex2vec embeddings for community detection in simplicial complexes.
\newblock arXiv:1906.09068.

\bibitem[{Bodnar et~al.(2021)Bodnar, Frasca, Wang, Otter, Montufar, Lio, and Bronstein}]{SWL2021}
Bodnar, C.; Frasca, F.; Wang, Y.; Otter, N.; Montufar, G.~F.; Lio, P.; and Bronstein, M. 2021.
\newblock Weisfeiler and lehman go topological: Message passing simplicial networks.
\newblock In \emph{International Conference on Machine Learning (ICML)}, 1026--1037.

\bibitem[{Bomze et~al.(1999)Bomze, Budinich, Pardalos, and Pelillo}]{maxClique1999}
Bomze, I.~M.; Budinich, M.; Pardalos, P.~M.; and Pelillo, M. 1999.
\newblock The maximum clique problem.
\newblock In \emph{Handbook of combinatorial optimization}, 1--74. Springer.

\bibitem[{Borgwardt et~al.(2005)Borgwardt, Ong, Sch{\"o}nauer, Vishwanathan, Smola, and Kriegel}]{data:proteins2005}
Borgwardt, K.~M.; Ong, C.~S.; Sch{\"o}nauer, S.; Vishwanathan, S.; Smola, A.~J.; and Kriegel, H.-P. 2005.
\newblock Protein function prediction via graph kernels.
\newblock \emph{Bioinformatics}, 21(suppl\_1): i47--i56.

\bibitem[{Bron and Kerbosch(1973)}]{find_cliques1973}
Bron, C.; and Kerbosch, J. 1973.
\newblock Algorithm 457: Finding All Cliques of an Undirected Graph.
\newblock \emph{Commun. ACM}, 16(9): 575–577.

\bibitem[{Cai, F{\"u}rer, and Immerman(1992)}]{CFI1992}
Cai, J.; F{\"u}rer, M.; and Immerman, N. 1992.
\newblock An optimal lower bound on the number of variables for graph identification.
\newblock \emph{Combinatorica}, 12(4): 398--140.

\bibitem[{Centola(2010)}]{SocialNet2010}
Centola, D. 2010.
\newblock The Spread of Behavior in an Online Social Network Experiment.
\newblock \emph{Science}, 329(5996): 1194--1197.

\bibitem[{Chen, Gel, and Poor(2022)}]{BScNets}
Chen, Y.; Gel, Y.~R.; and Poor, H.~V. 2022.
\newblock BScNets: Block Simplicial Complex Neural Networks.
\newblock volume~36, 6333--6341.

\bibitem[{Chiba and Nishizeki(1985)}]{chiba1985arboricity}
Chiba, N.; and Nishizeki, T. 1985.
\newblock Arboricity and subgraph listing algorithms.
\newblock \emph{SIAM J. Comput.}, 14(1): 210--223.

\bibitem[{Chien et~al.(2020)Chien, Peng, Li, and Milenkovic}]{GPRGNN}
Chien, E.; Peng, J.; Li, P.; and Milenkovic, O. 2020.
\newblock Adaptive Universal Generalized PageRank Graph Neural Network.
\newblock In \emph{International Conference on Learning Representations}.

\bibitem[{de~Haan, Cohen, and Welling(2020)}]{natural_GN_2020}
de~Haan, P.; Cohen, T.~S.; and Welling, M. 2020.
\newblock Natural graph networks.
\newblock \emph{Advances in neural information processing systems}, 33: 3636--3646.

\bibitem[{Debnath et~al.(1991)Debnath, Lopez~de Compadre, Debnath, Shusterman, and Hansch}]{data:MUTAG}
Debnath, A.~K.; Lopez~de Compadre, R.~L.; Debnath, G.; Shusterman, A.~J.; and Hansch, C. 1991.
\newblock Structure-activity relationship of mutagenic aromatic and heteroaromatic nitro compounds. correlation with molecular orbital energies and hydrophobicity.
\newblock \emph{Journal of medicinal chemistry}, 34(2): 786--797.

\bibitem[{Defferrard, Bresson, and Vandergheynst(2016)}]{ChebNet}
Defferrard, M.; Bresson, X.; and Vandergheynst, P. 2016.
\newblock Convolutional neural networks on graphs with fast localized spectral filtering.
\newblock \emph{Advances in neural information processing systems}, 29: 3838--3845.

\bibitem[{Ebli, Defferrard, and Spreemann(2020)}]{SNN2020}
Ebli, S.; Defferrard, M.; and Spreemann, G. 2020.
\newblock Simplicial Neural Networks.
\newblock In \emph{NeurIPS 2020 Workshop on Topological Data Analysis and Beyond}.

\bibitem[{Fey and Lenssen(2019)}]{graphwithPyTorch2019}
Fey, M.; and Lenssen, J.~E. 2019.
\newblock Fast graph representation learning with PyTorch Geometric.
\newblock arXiv:1903.02428.

\bibitem[{Ganmor, Segev, and Schneidman(2011)}]{BrainNet2011}
Ganmor, E.; Segev, R.; and Schneidman, E. 2011.
\newblock Sparse low-order interaction network underlies a highly correlated and learnable neural population code.
\newblock \emph{Proceedings of the National Academy of Sciences}, 108(23): 9679--9684.

\bibitem[{Gao et~al.(2022)Gao, Zhang, Lin, Zhao, Du, and Zou}]{HypergraphLearning22review}
Gao, Y.; Zhang, Z.; Lin, H.; Zhao, X.; Du, S.; and Zou, C. 2022.
\newblock Hypergraph Learning: Methods and Practices.
\newblock \emph{IEEE Transactions on Pattern Analysis and Machine Intelligence}, 44(5): 2548--2566.

\bibitem[{G{\"a}rtner, Flach, and Wrobel(2003)}]{RWK2003}
G{\"a}rtner, T.; Flach, P.; and Wrobel, S. 2003.
\newblock On graph kernels: Hardness results and efficient alternatives.
\newblock In \emph{Learning Theory and Kernel Machines}, 129--143. Springer.

\bibitem[{Gasteiger, Bojchevski, and G{\"u}nnemann(2019)}]{APPNP}
Gasteiger, J.; Bojchevski, A.; and G{\"u}nnemann, S. 2019.
\newblock Predict then Propagate: Graph Neural Networks meet Personalized PageRank.
\newblock In \emph{International Conference on Learning Representations}.

\bibitem[{Grilli et~al.(2017)Grilli, Barab{\'a}s, Michalska-Smith, and Allesina}]{Ecology2017}
Grilli, J.; Barab{\'a}s, G.; Michalska-Smith, M.~J.; and Allesina, S. 2017.
\newblock Higher-order interactions stabilize dynamics in competitive network models.
\newblock \emph{Nature}, 548(7666): 210--213.

\bibitem[{Hacker(2020)}]{k-simplex2vec}
Hacker, C. 2020.
\newblock k-simplex2vec: a simplicial extension of node2vec.
\newblock arXiv:2010.05636.

\bibitem[{Hajij et~al.(2022)Hajij, Zamzmi, Papamarkou, Maroulas, and Cai}]{GMPS22}
Hajij, M.; Zamzmi, G.; Papamarkou, T.; Maroulas, V.; and Cai, X. 2022.
\newblock Simplicial complex representation learning.
\newblock In \emph{Machine Learning on Graphs (MLoG) Workshop at 15th ACM International WSDM Conference}.

\bibitem[{Hatcher(2002)}]{Hodge_Hatcher}
Hatcher, A. 2002.
\newblock \emph{Algebraic topology}.
\newblock Cambridge University Press.

\bibitem[{He et~al.(2021)He, Wei, Xu et~al.}]{BernNet}
He, M.; Wei, Z.; Xu, H.; et~al. 2021.
\newblock Bernnet: Learning arbitrary graph spectral filters via bernstein approximation.
\newblock \emph{Advances in Neural Information Processing Systems}, 34: 14239--14251.

\bibitem[{Hu et~al.(2020)Hu, Fey, Zitnik, Dong, Ren, Liu, Catasta, and Leskovec}]{data:ogbn_arxiv}
Hu, W.; Fey, M.; Zitnik, M.; Dong, Y.; Ren, H.; Liu, B.; Catasta, M.; and Leskovec, J. 2020.
\newblock Open graph benchmark: Datasets for machine learning on graphs.
\newblock \emph{Advances in neural information processing systems}, 33: 22118--22133.

\bibitem[{Kendall(1938)}]{kendall1938new}
Kendall, M.~G. 1938.
\newblock A new measure of rank correlation.
\newblock \emph{Biometrika}, 30(1/2): 81--93.

\bibitem[{Kipf and Welling(2017)}]{GCN}
Kipf, T.~N.; and Welling, M. 2017.
\newblock Semi-Supervised Classification with Graph Convolutional Networks.
\newblock In \emph{ICLR}.

\bibitem[{Lee, Ji, and Tay(2022)}]{SGAT}
Lee, S.~H.; Ji, F.; and Tay, W.~P. 2022.
\newblock SGAT: Simplicial Graph Attention Network.
\newblock In \emph{Proceedings of the Thirty-First International Joint Conference on Artificial Intelligence (IJCAI-22)}, 3192--3200.

\bibitem[{Li et~al.(2022)Li, Zhu, Cheng, Shan, Luo, Li, and Qian}]{GloGNN}
Li, X.; Zhu, R.; Cheng, Y.; Shan, C.; Luo, S.; Li, D.; and Qian, W. 2022.
\newblock Finding global homophily in graph neural networks when meeting heterophily.
\newblock In \emph{International Conference on Machine Learning}, 13242--13256. PMLR.

\bibitem[{Lim and Benson(2021)}]{data:genius}
Lim, D.; and Benson, A.~R. 2021.
\newblock Expertise and dynamics within crowdsourced musical knowledge curation: A case study of the genius platform.
\newblock In \emph{Proceedings of the International AAAI Conference on Web and Social Media}, volume~15, 373--384.

\bibitem[{Luan et~al.(2022)Luan, Hua, Lu, Zhu, Zhao, Zhang, Chang, and Precup}]{ACM-GCN}
Luan, S.; Hua, C.; Lu, Q.; Zhu, J.; Zhao, M.; Zhang, S.; Chang, X.-W.; and Precup, D. 2022.
\newblock Revisiting heterophily for graph neural networks.
\newblock \emph{Advances in neural information processing systems}, 35: 1362--1375.

\bibitem[{Maron et~al.(2019)Maron, Ben-Hamu, Serviansky, and Lipman}]{PPGNs_maron2019}
Maron, H.; Ben-Hamu, H.; Serviansky, H.; and Lipman, Y. 2019.
\newblock Provably powerful graph networks.
\newblock \emph{Advances in neural information processing systems}, 32.

\bibitem[{Maron et~al.(2018)Maron, Ben-Hamu, Shamir, and Lipman}]{IGN_2018}
Maron, H.; Ben-Hamu, H.; Shamir, N.; and Lipman, Y. 2018.
\newblock Invariant and Equivariant Graph Networks.
\newblock In \emph{International Conference on Learning Representations}.

\bibitem[{Milgram(1967)}]{SmallWorld1967}
Milgram, S. 1967.
\newblock The small world problem.
\newblock \emph{Psychology today}, 2(1): 60--67.

\bibitem[{Morris et~al.(2019)Morris, Ritzert, Fey, Hamilton, Lenssen, Rattan, and Grohe}]{kWL2019}
Morris, C.; Ritzert, M.; Fey, M.; Hamilton, W.~L.; Lenssen, J.~E.; Rattan, G.; and Grohe, M. 2019.
\newblock Weisfeiler and Leman Go Neural: Higher-Order Graph Neural Networks.
\newblock \emph{Proceedings of the AAAI Conference on Artificial Intelligence}, 33(01): 4602--4609.

\bibitem[{Neumann et~al.(2016)Neumann, Garnett, Bauckhage, and Kersting}]{PK_2016neumann}
Neumann, M.; Garnett, R.; Bauckhage, C.; and Kersting, K. 2016.
\newblock Propagation kernels: efficient graph kernels from propagated information.
\newblock \emph{Machine Learning}, 102: 209--245.

\bibitem[{Orsini et~al.(2015)Orsini, Dankulov, Colomer-de Sim{\'o}n, Jamakovic, Mahadevan, Vahdat, Bassler, Toroczkai, Bogun{\'a}, Caldarelli et~al.}]{nullmodel}
Orsini, C.; Dankulov, M.~M.; Colomer-de Sim{\'o}n, P.; Jamakovic, A.; Mahadevan, P.; Vahdat, A.; Bassler, K.~E.; Toroczkai, Z.; Bogun{\'a}, M.; Caldarelli, G.; et~al. 2015.
\newblock Quantifying randomness in real networks.
\newblock \emph{Nature communications}, 6(1): 8627.

\bibitem[{Pei et~al.(2020)Pei, Wei, Chang, Lei, and Yang}]{Pei2020Geom-GCN}
Pei, H.; Wei, B.; Chang, K. C.-C.; Lei, Y.; and Yang, B. 2020.
\newblock Geom-GCN: Geometric Graph Convolutional Networks.
\newblock In \emph{International Conference on Learning Representations}.

\bibitem[{Roddenberry, Glaze, and Segarra(2021)}]{SCoNe21Roddenberry}
Roddenberry, T.~M.; Glaze, N.; and Segarra, S. 2021.
\newblock Principled Simplicial Neural Networks for Trajectory Prediction.
\newblock In \emph{Proceedings of the 38th International Conference on Machine Learning}, volume 139, 9020--9029. PMLR.

\bibitem[{Rozemberczki, Allen, and Sarkar(2021)}]{Rozemberczki2021}
Rozemberczki, B.; Allen, C.; and Sarkar, R. 2021.
\newblock Multi-scale attributed node embedding.
\newblock \emph{Journal of Complex Networks}, 9(2): 1--22.

\bibitem[{Schaub et~al.(2020)Schaub, Benson, Horn, Lippner, and Jadbabaie}]{hodge_schaub2020random}
Schaub, M.~T.; Benson, A.~R.; Horn, P.; Lippner, G.; and Jadbabaie, A. 2020.
\newblock Random walks on simplicial complexes and the normalized Hodge 1-Laplacian.
\newblock \emph{SIAM Review}, 62(2): 353--391.

\bibitem[{Shchur et~al.(2018)Shchur, Mumme, Bojchevski, and G{\"{u}}nnemann}]{Shchur2018}
Shchur, O.; Mumme, M.; Bojchevski, A.; and G{\"{u}}nnemann, S. 2018.
\newblock Pitfalls of Graph Neural Network Evaluation.
\newblock \emph{CoRR}, abs/1811.05868.

\bibitem[{Shervashidze et~al.(2011)Shervashidze, Schweitzer, Van~Leeuwen, Mehlhorn, and Borgwardt}]{WL_kernel_2011shervashidze}
Shervashidze, N.; Schweitzer, P.; Van~Leeuwen, E.~J.; Mehlhorn, K.; and Borgwardt, K.~M. 2011.
\newblock Weisfeiler-lehman graph kernels.
\newblock \emph{Journal of Machine Learning Research}, 12(9).

\bibitem[{Shervashidze et~al.(2009)Shervashidze, Vishwanathan, Petri, Mehlhorn, and Borgwardt}]{GK_2009efficient}
Shervashidze, N.; Vishwanathan, S.; Petri, T.; Mehlhorn, K.; and Borgwardt, K. 2009.
\newblock Efficient graphlet kernels for large graph comparison.
\newblock In \emph{Artificial intelligence and statistics}, 488--495. PMLR.

\bibitem[{Shuman et~al.(2013)Shuman, Narang, Frossard, Ortega, and Vandergheynst}]{graphFourier2013}
Shuman, D.~I.; Narang, S.~K.; Frossard, P.; Ortega, A.; and Vandergheynst, P. 2013.
\newblock The emerging field of signal processing on graphs: Extending high-dimensional data analysis to networks and other irregular domains.
\newblock \emph{IEEE Signal Processing Magazine}, 30(3): 83--98.

\bibitem[{Sinha et~al.(2015)Sinha, Shen, Song, Ma, Eide, Hsu, and Wang}]{data:MAG-His-Geo}
Sinha, A.; Shen, Z.; Song, Y.; Ma, H.; Eide, D.; Hsu, B.-J.~P.; and Wang, K. 2015.
\newblock An Overview of Microsoft Academic Service ({MAS}) and Applications.
\newblock In \emph{Proceedings of the 24th International Conference on World Wide Web}. {ACM} Press.

\bibitem[{Toivonen et~al.(2003)Toivonen, Srinivasan, King, Kramer, and Helma}]{data:PTC}
Toivonen, H.; Srinivasan, A.; King, R.~D.; Kramer, S.; and Helma, C. 2003.
\newblock Statistical evaluation of the predictive toxicology challenge 2000--2001.
\newblock \emph{Bioinformatics}, 19(10): 1183--1193.

\bibitem[{Traud, Mucha, and Porter(2012)}]{data:penn94}
Traud, A.~L.; Mucha, P.~J.; and Porter, M.~A. 2012.
\newblock Social structure of facebook networks.
\newblock \emph{Physica A: Statistical Mechanics and its Applications}, 391(16): 4165--4180.

\bibitem[{Veli{\v{c}}kovi{\'c} et~al.(2018)Veli{\v{c}}kovi{\'c}, Cucurull, Casanova, Romero, Li{\`o}, and Bengio}]{GAT2018}
Veli{\v{c}}kovi{\'c}, P.; Cucurull, G.; Casanova, A.; Romero, A.; Li{\`o}, P.; and Bengio, Y. 2018.
\newblock Graph Attention Networks.
\newblock In \emph{International Conference on Learning Representations (ICLR)}.

\bibitem[{Wang and Zhang(2022)}]{LinearGNN}
Wang, X.; and Zhang, M. 2022.
\newblock How Powerful are Spectral Graph Neural Networks.
\newblock In \emph{International Conference on Machine Learning (ICML)}.

\bibitem[{Weisfeiler and Leman(1968)}]{WLtest}
Weisfeiler, B.; and Leman, A. 1968.
\newblock The reduction of a graph to canonical form and the algebra which appears therein.
\newblock \emph{NTI, Series}, 2(9): 12--16.

\bibitem[{Wu et~al.(2019)Wu, Souza, Zhang, Fifty, Yu, and Weinberger}]{SGC}
Wu, F.; Souza, A.; Zhang, T.; Fifty, C.; Yu, T.; and Weinberger, K. 2019.
\newblock Simplifying Graph Convolutional Networks.
\newblock In Chaudhuri, K.; and Salakhutdinov, R., eds., \emph{Proceedings of the 36th International Conference on Machine Learning}, volume~97 of \emph{Proceedings of Machine Learning Research}, 6861--6871. PMLR.

\bibitem[{Xu et~al.(2019)Xu, Hu, Leskovec, and Jegelka}]{GIN2019}
Xu, K.; Hu, W.; Leskovec, J.; and Jegelka, S. 2019.
\newblock How powerful are graph neural networks?
\newblock In \emph{International Conference on Learning Representations (ICLR)}.

\bibitem[{Yan et~al.(2021)Yan, Hashemi, Swersky, Yang, and Koutra}]{GGCN}
Yan, Y.; Hashemi, M.; Swersky, K.; Yang, Y.; and Koutra, D. 2021.
\newblock Two sides of the same coin: Heterophily and oversmoothing in graph convolutional neural networks.
\newblock arXiv:2102.06462.

\bibitem[{Yanardag and Vishwanathan(2015)}]{data:IMDB-REDDIT}
Yanardag, P.; and Vishwanathan, S. 2015.
\newblock Deep Graph Kernels.
\newblock In \emph{Proceedings of the 21th ACM SIGKDD International Conference on Knowledge Discovery and Data Mining}, KDD '15, 1365–1374. New York, NY, USA: Association for Computing Machinery.
\newblock ISBN 9781450336642.

\bibitem[{Yang, Isufi, and Leus(2022)}]{SCNN2022}
Yang, M.; Isufi, E.; and Leus, G. 2022.
\newblock Simplicial convolutional neural networks.
\newblock In \emph{ICASSP}, 8847--8851.

\bibitem[{Yang et~al.(2023)Yang, Zhou, Liu, and L{\"u}}]{YRM2023}
Yang, R.; Zhou, F.; Liu, B.; and L{\"u}, L. 2023.
\newblock A generalized simplicial model and its application.
\newblock arXiv:2309.02851.

\bibitem[{Yang, Cohen, and Salakhudinov(2016)}]{Yang2016CiteGraph}
Yang, Z.; Cohen, W.; and Salakhudinov, R. 2016.
\newblock Revisiting Semi-Supervised Learning with Graph Embeddings.
\newblock In \emph{Proceedings of The 33rd International Conference on Machine Learning}, volume~48 of \emph{Proceedings of Machine Learning Research}, 40--48. New York, New York, USA: PMLR.

\bibitem[{Zeng et~al.(2023{\natexlab{a}})Zeng, Huang, Ren, and L{\"u}}]{HoRW}
Zeng, Y.; Huang, Y.; Ren, X.-L.; and L{\"u}, L. 2023{\natexlab{a}}.
\newblock Identifying vital nodes through augmented random walks on higher-order networks.
\newblock arXiv:2305.06898.

\bibitem[{Zeng et~al.(2023{\natexlab{b}})Zeng, Huang, Wu, and L{\"u}}]{ISMnet}
Zeng, Y.; Huang, Y.; Wu, Q.; and L{\"u}, L. 2023{\natexlab{b}}.
\newblock Influential Simplices Mining via Simplicial Convolutional Network.
\newblock arXiv:2307.05841.

\bibitem[{Zeng et~al.(2023{\natexlab{c}})Zeng, Liu, Zhou, and L{\"u}}]{zeng2023hyper}
Zeng, Y.; Liu, B.; Zhou, F.; and L{\"u}, L. 2023{\natexlab{c}}.
\newblock Hyper-Null Models and Their Applications.
\newblock \emph{Entropy}, 25(10): 1390.

\bibitem[{Zhang et~al.(2023)Zhang, Yan, He, Li, and Chu}]{DRGCN}
Zhang, L.; Yan, X.; He, J.; Li, R.; and Chu, W. 2023.
\newblock DRGCN: Dynamic Evolving Initial Residual for Deep Graph Convolutional Networks.
\newblock arXiv:2302.05083.

\bibitem[{Zhang et~al.(2018)Zhang, Cui, Neumann, and Chen}]{DGCNN_zhang2018}
Zhang, M.; Cui, Z.; Neumann, M.; and Chen, Y. 2018.
\newblock An end-to-end deep learning architecture for graph classification.
\newblock In \emph{Proceedings of the AAAI conference on artificial intelligence}, volume~32.

\bibitem[{Zhu et~al.(2020)Zhu, Yan, Zhao, Heimann, Akoglu, and Koutra}]{homo2020}
Zhu, J.; Yan, Y.; Zhao, L.; Heimann, M.; Akoglu, L.; and Koutra, D. 2020.
\newblock Beyond homophily in graph neural networks: Current limitations and effective designs.
\newblock \emph{Advances in Neural Information Processing Systems}, 33: 7793--7804.

\end{thebibliography}
\appendix

\onecolumn
\begin{center}
{\Large \bf Appendix} 
\end{center}

\section{Spectral analysis for flower-petals algebraic descriptions}
\label{appendix: spectral}
In this section, we provide a theoretical analysis of the spectral properties of the flower-petals (FP) adjacency and Laplacian matrices.

We divide Theorem \ref{theorem: semi-positive} into two separate lemmas and prove them individually.

\begin{lemma}[Non-negativity of $\tilde{\mathcal{A}}_p$]
\label{lemma: semi-positive_A}
The flower-petals adjacency matrices $\tilde{\mathcal{A}}_p$ are symmetric positive semidefinite.
\end{lemma}

\begin{proof}
By considering the bilinear form $x^\top \tilde{\mathcal{A}}_p x$, it can be easily drawn that $\tilde{\mathcal{A}}_p$ are positive semidefinite:
\begin{equation}
    \begin{split}
        x^\top \tilde{\mathcal{A}}_p x   
        & = \frac{1}{p+1} x^\top D_{p,v}^{-1/2} \mathcal{H}_p \mathcal{H}_p^\top D_{p,v}^{-1/2} x\\
        & = \frac{1}{p+1} \left(x^\top D_{p,v}^{-1/2} \mathcal{H}_p\right) \left(\mathcal{H}_p^\top D_{p,v}^{-1/2} x\right) \\
        & \geq 0.
    \end{split}
\end{equation}
\end{proof}

\begin{lemma}[Non-negativity of $\mathcal{L}_p $]
\label{lemma: semi-positive_L}
The flower-petals Laplacian matrices $\mathcal{L}_p $ are symmetric positive semidefinite.
\end{lemma}

\begin{proof}
We first introduce a matrix $G_{\sigma_p}$ to support the proof,
\begin{equation}
    G_{\sigma_p}(i,j) = \left\{
        \begin{array}{cl}
        p, & v_i\in\sigma_p, v_j \in \sigma_p, i = j \\ 
        -1,& v_i\in\sigma_p, v_j \in \sigma_p, i \neq j \\
        0, & \text { otherwise } 
        \end{array}\right. .
\end{equation}

Let $\sigma_p = \left[v_{i_1},v_{i_2},\cdots,v_{i_{(p+1)}}\right]$ be an arbitrary $p$-simplex, and the notation $x=\left(x_1,x_2,\cdots,x_n\right)^\top$ denotes an arbitrary vector in $\mathbb{R}^n$. By considering the bilinear form $x^\top G_{\sigma_p}x$, we see that $G_{\sigma_p}$ is positive semidefinite:
\begin{equation}
    \begin{split}
        x^\top G_{\sigma_p}x  
        =& x_{i_1}\left(px_{i_1}-x_{i_2}-\cdots-x_{i_{p+1}}\right) + \\
        &   x_{i_2}\left(px_{i_2}-x_{i_1}-\cdots-x_{i_{p+1}}\right) + \cdots\\
         =& (p+1)\sum_{t=1}^{p+1}x_{i_t}^2 - \left(\sum_{t=1}^{p+1}x_{i_t}\right)^2\\
         \geq& 0.
    \end{split}
\end{equation}
The last inequality holds exploiting the Cauchy-Schwarz inequality.

We write the higher-order FP Laplacian as a sum over the $p$-simplices that
\begin{equation}
    \begin{split}
       \mathcal{L}_p  
        &= I  - \tilde{\mathcal{A}}_p \\
        & = \frac{1}{p+1}D_{p,v}^{-1/2}\left[(p+1)D_{p,v}-\mathcal{H}_p \mathcal{H}_p^\top \right]D_{p,v}^{-1/2}\\
        & =\frac{1}{p+1}D_{p,v}^{-1/2} \left(\sum_{\sigma_p \in \mathcal{K}_p} G_{\sigma_p}\right) D_{p,v}^{-1/2}.
    \end{split}
\end{equation}

We now consider the bilinear form of $\mathcal{L}_p$,
\begin{equation}
    \begin{split}
        y^\top\mathcal{L}_p y  
        &= \frac{1}{p+1} \left(y^\top D_{p,v}^{-1/2}\right) \sum_{\sigma_p \in \mathcal{K}_p} G_{\sigma_p} \left(D_{p,v}^{-1/2}y\right)  \\
        & =\frac{1}{p+1}D_{p,v}^{-1/2} \sum_{\sigma_p \in \mathcal{K}_p} \left(x^\top G_{\sigma_p} x\right) D_{p,v}^{-1/2}  \\
        & \geq 0. 
    \end{split}
\end{equation}
Here, $x=D_{p,v}^{-1/2}y$. It follows that $\mathcal{L}_p$ are positive semidefinite for arbitrary $p$.
\end{proof}

According to Lemma \ref{lemma: semi-positive_A} and Lemma \ref{lemma: semi-positive_L}, it can be directedly drawn that $0 \leq \lambda(\tilde{\mathcal{A}}_p) = \lambda\left(I-\mathcal{L}_p\right)\leq 1$.
In a similar way, one can conclude that $0 \leq  \lambda\left(\mathcal{L}_p\right)\leq 1$.

\begin{lemma}
0 is an eigenvalue of the higher-order flower-petals Laplacian matrices $\mathcal{L}_p$.
\end{lemma}

\begin{proof}
Let $\boldsymbol{1}=(1,1,\cdots,1)^\top$ be an all-one vector. According to the definition of $\mathcal{H}_p$, we can obtain that $\mathcal{H}_p^\top\boldsymbol{1}=(p+1)\boldsymbol{1}$ and $\mathcal{H}_p \boldsymbol{1} = D_{p,v}$. It follows that
\begin{equation}
    \begin{split}
        \mathcal{L}_p\boldsymbol{1} 
        & = \left(I-\frac{1}{p+1}D_{p,v}^{-1/2}\mathcal{H}_p{\mathcal{H}_p}^\top D_{p,v}^{-1/2}\right)\boldsymbol{1} \\
        & = D_{p,v}^{-1/2}\left(D_{p,v}-\frac{1}{p+1}\mathcal{H}_p{\mathcal{H}_p}^\top\boldsymbol{1}\right)D_{p,v}^{-1/2}\\
        & =0.
    \end{split}
\end{equation}

Therefore, $\boldsymbol{1}$ is an eigenvector of $\mathcal{L}_p$ associated with eigenvalue 0, which is also the smallest eigenvalue from the non-negativity of $\mathcal{L}_p$.
\end{proof}

\section{Expressive power analysis details}
\label{appendix: expressive_power_analysis}


To demonstrate the superior expressiveness of the proposed model, we begin by offering foundational background information in Appendix \ref{appendix: clique_complex_lift}. 
Subsequently, we introduce the Weisfeiler-Lehman test and its higher-order versions in Appendices \ref{appendix: WL} and \ref{appendix:HWL_SHWL}, respectively. Finally, Appendix \ref{appendix: proofs_HWL_SHWL} presents the formal proof.

\subsection{Clique complex lifting.} 
\label{appendix: clique_complex_lift}
We can obtain a clique complex, a kind of SCs, by extracting all cliques from a graph and regarding them as simplices. 
This implies that an empty triangle (owning $\left[v_1,v_2\right]$, $\left[v_1,v_3\right]$, $\left[v_2,v_3\right]$ but without $\left[v_1,v_2,v_3\right]$) cannot occur in clique complexes.

The transition from a pairwise graph to a clique complex, termed the clique complex lifting transition, enables us to study pairwise graphs from simplicial perspectives. 
We defer the complexity discussion of this operation in Appendix \ref{appendix: large}.

\subsection{Introduction to the Weisfeiler-Lehman test}
\label{appendix: WL}

The Weisfeiler-Lehman (WL) graph isomorphism test \cite{WLtest} provides a well-studied framework for the unique assignment of node labels. An intrinsic theoretical connection has been uncovered between the WL test and spatial GNNs \cite{GIN2019, kWL2019}. Here, we first present a brief introduction to the WL test.

\begin{algorithm}[ht] 
\caption{Weisfeiler-Lehman Test}
\textbf{Input}: Graph $\mathcal{G}$, initial node coloring $\left\{c^{(0)}(v_1), c^{(0)}(v_2), \cdots, c^{(0)}(v_n)\right\}$ \\
\textbf{Output}:~~Final node coloring $\left\{c^{(T)}(v_1), \cdots, c^{(T)}(v_n)\right\}$.
\vspace{-1em} %
\begin{algorithmic}[1] 
\STATE $t \leftarrow 0$ 
\REPEAT
    \FOR{$v$ in $\mathcal{V}$}
    \STATE $c^{(t+1)}(v) = \operatorname{Hash}\left(\multiset{c^{(t)}(u)|u \in N(v)\cup \{v\} }\right)$ 
    \ENDFOR
    \STATE $t \leftarrow t+1$
\UNTIL{stable node coloring is reached}
\end{algorithmic}
\label{alg:WL-test}
\end{algorithm}

Two graphs $\mathcal{G}$ and $\mathcal{G}'$ are considered to be isomorphic if there exists an edge-preserving bijection $\psi: \mathcal{V} \to \mathcal{V}'$, i.e., $(u,v)\in \mathcal{E}(\mathcal{G})$ if and only if $\left(\psi(u), \psi(v)\right)\in\mathcal{E}(\mathcal{G}')$. 
Graph isomorphism determination is an NP-hard problem, and the Weisfeiler-Lehman test, or $k$-WL, is a family of heuristic algorithms for testing graph isomorphism \cite{kWL2019}. $k$-WL constructs a coloring $c$ of the tuples of $k$ nodes, that is $c: \mathcal{V}^k \to \Sigma$ with arbitrary codomain $\Sigma$.

We now describe the WL test for graph $\mathcal{G}$ with coloring $c$.  In each iteration $t$, the WL algorithm updates a new coloring according to the rule
\begin{equation}
    c^{(t+1)}(v) = \operatorname{Hash}\left( \multiset{c^{(t)}(u)| u \in N(v)\cup \{v\} } \right),
\end{equation}
where $\operatorname{Hash}(\cdot)$ bijectively maps different multi-set inputs $\multiset{\cdot}$ to a unique color in $\Sigma$, and $N(v)$ presents the set of nodes adjacent to node $v$ in $\mathcal{G}$, i.e., $N(v)=\left\{u \in \mathcal{V} | (v,u)\in \mathcal{E}\right\}$. The iteration is finished until stable node coloring is reached, i.e., $c^{(t+1)}=c^{(t)}$, and termination is guaranteed within $\max \{|\mathcal{V}|, |\mathcal{E}|\}$ iterations.

To distinguish two graphs $\mathcal{G}$ and $\mathcal{G}'$ in terms of isomorphism, we perform the WL test on both $\mathcal{G}$ and $\mathcal{G}'$ in parallel, and these two graphs are non-isomorphic if their color histograms are different in the iteration.

The $k$-WL test $(k \geq 2)$ is a generalized version of the WL test that colors the tuples $\mathcal{V}^k$ instead of $\mathcal{V}$. The algorithms with larger $k$ are more powerful in distinguishing non-isomorphic graphs. However, it is noted that although the $k$-WL test is a powerful heuristic, it is not guaranteed to be effective in all cases \cite{CFI1992}. Without causing ambiguity,  we abbreviate the $1$-WL test as the WL test to unload the notation burden.


\subsection{Higher-order version of WL test}
\label{appendix:HWL_SHWL}

The proposed higher-order Weisfeiler-Lehman test, termed HWL, extends the WL test to simplicial complexes. 
SHWL, a simplified version of HWL, further reduces the coloring rule.
We can relate the expressive power of HWL, SHWL and WL by clique complex lifting transition, and it has been found that both HWL and SHWL are more powerful than Weisfeiler-Lehman in terms of expressive power.

We outline below the steps of HWL in a given simplicial complex $\mathcal{K}$ for example.

(a) Assign the same initial color $c^{(0)}(\sigma)$ to each simplex $\sigma$ in $\mathcal{K}$.

(b) Given the color  $c^{(t)}(\sigma)$ of simplex $\sigma$ at the $t$-th iteration. We will refine its color by injectively mapping neighborhood colors in the flower-petals model according to
\begin{equation}
    c^{(t+1)}(\sigma) = \operatorname{Hash}\left(c^{(t)}(\sigma), \multiset{c^{(t)}(\tau)|\tau \in \mathcal{N}(\sigma) } \right).
\end{equation}
Here, the Hash function maps different multi-set inputs to different colors, $\mathcal{N}(\sigma)=\bigcup_p \mathcal{N}_p(\sigma)$, and $\mathcal{N}_p$ denotes the set of neighbors of $\sigma$ in the bipartite graph $\mathcal{G}_p$.

(c) The algorithm stops until a stable coloring is reached. Two simplicial complexes are considered non-isomorphic if their color histograms are different at any iteration. Otherwise, the test is inconclusive.

\begin{theorem}
\label{the:HWL}
HWL with clique complex lifting is strictly more powerful than the Weisfeiler-Lehman (WL) test.
\end{theorem}

The proof is deferred to Appendix \ref{appendix: proofs_HWL_SHWL}.

Moreover, we can simplify the coloring rule for simplices whose order is larger than zero to nonvanishing linear functions, i.e.,
\begin{equation}
    c^{(t+1)}(\sigma) = \operatorname{Linear}\left(c^{(t)}(\sigma), \multiset{c^{(t)}(\tau)|\tau \in \mathcal{N}(\sigma) } \right).
\end{equation}
Here, $dim(\sigma)>0$ and $\operatorname{Linear}(\cdot)$ is a nonvanishing linear function requiring that the coefficients of the polynomial never be zero. Other steps are the same as in HWL, and the simplified version of HWL, referred to as SHWL, is shown in Algorithm \ref{alg:SHWL}.

We find that SHWL is still more powerful than the WL test (refer to Theorem \ref{the:SHWL}).

\begin{algorithm}[htp] 
\caption{Simplified Higher-order WL Test (SHWL)}
\textbf{Input}: Simplicial complex $\mathcal{K}$, initial simplicial coloring $\left\{c^{(0)}(\sigma_1), c^{(0)}(\sigma_2), \cdots \right\}$ \\
\textbf{Output}: Final node coloring $\left\{c^{(T)}(v_1), \cdots, c^{(T)}(v_n)\right\}$.
\vspace{-1em} 
\begin{algorithmic}[1]
\STATE $t \leftarrow 0$ 
\REPEAT
    \FOR{$v$ in $\mathcal{V}$}
    \STATE $c^{(t+1)}(v) = \operatorname{Hash}\left(c^{(t)}(v), \multiset{c^{(t)}(\tau)|\tau \in \mathcal{N}(v) } \right)$ 
    \ENDFOR
    \FOR{$\sigma$ in $\mathcal{K}_1 \cup \mathcal{K}_2 \cup \cdots$}
    \STATE $c^{(t+1)}(\sigma) = \operatorname{Linear}\left(c^{(t)}(\sigma), \multiset{c^{(t)}(v)| v \in \mathcal{N}(\sigma) } \right)$ 
    \ENDFOR
    \STATE $t \leftarrow t+1$
\UNTIL{stable simplicial coloring is reached}
\end{algorithmic}
\label{alg:SHWL}
\end{algorithm}

In dense graphs, the number of total simplices is much larger than the node number. It is not necessary to explicitly figure out messages for higher-order structures when tackling the most common node-level tasks. 
Hence, we can perform the two message-passing steps simultaneously. By replacing the colors with continuous feature vectors in the SHWL test and replacing the Hash function with a neural network layer-like differentiable function with learnable parameters, we can obtain 
\begin{equation}
    X^{(k+1)} = \sigma \left(\tilde{\mathcal{A}}_p X^{(k)}\Theta_p^{(k)} \right).
\end{equation}
Here, $X^{(k)} \in \mathbb{R}^{n\times d}$ is the signal at the $k$-th layer, $X^{(0)}=X$ presents the initial node features, and $\sigma(\cdot)$ is differentiable non-linear transition function. Following \cite{SGC}, the non-linear transition functions between the layers are removed.
Besides, by linearly combining the results from different layers, we can further obtain that 
\begin{equation}
   Y_p = \sum_{k=0}^K\gamma_{p,k}\tilde{\mathcal{A}}_p^k X\Theta_p.
\end{equation}
Here, $\Theta_p = \Theta_p^{(1)}\Theta_p^{(2)}\cdots\Theta_p^{(K)}$ is reparameterized. 

We finally concatenate the results from different FP convolutions followed by a linear layer to reshape the output as follows
\begin{equation}
    Y =   \mathop{\Big|\!\Big|}\limits_{p=1}^P \left(\sum_{k=0}^K \gamma_{p,k} \tilde{\mathcal{A}_p^k} X \Theta_p \right) W .
\end{equation}

%
Therefore, the HiGCN model can  be interpreted as a neural version of the SHWL test. 
Theorem \ref{the:SHWL} suggests that the proposed HiGCN model endows with greater potential than the traditional GNNs.

\subsection{Proofs of HWL and SHWL theory}
\label{appendix: proofs_HWL_SHWL}

We begin by introducing some required notions and symbols. Note that although these results apply mainly to simplicial complexes, they also work for graphs since graphs can be viewed as special simplicial complexes. Besides, we can also obtain a unique simplex complex by taking clique complex lifting.

\begin{definition}
A simplicial coloring $c$ is a mapping that assigns a color from a specific color scheme to a simplicial complex $\mathcal{K}$ and its simplices $\sigma$. We notate this color as $c^{\mathcal{K}}(\sigma)$, or as $c(\sigma)$ when it does not cause ambiguity.
\end{definition}

\begin{definition}
Let $c$ be a simplicial coloring. Simplicial complexes $\mathcal{J}$, $\mathcal{K}$ are considered to be $c$-similar, represented by $c^{\mathcal{J}}=c^{\mathcal{K}}$, if the numbers of simplices with the same color in $\mathcal{J}$ and $\mathcal{K}$ are identical. Otherwise, we say $c^{\mathcal{J}} \neq c^{\mathcal{K}}$.
\end{definition}

\begin{definition}
A simplicial coloring $c$ refines another one $d$, noted as $c \sqsubseteq d$. Then for all simplicial complexes $\mathcal{J}$, $\mathcal{K}$ and all $\sigma \in \mathcal{J}, \tau \in \mathcal{K}$, $d^{\mathcal{J}}(\sigma) = d^{\mathcal{K}}(\tau)$ iff $c^{\mathcal{J}}(\sigma) = c^{\mathcal{K}}(\tau)$.
\end{definition}

If $d^{\mathcal{J}}(\sigma) \neq d^{\mathcal{K}}(\tau)$ for two simplicial colorings $c$ and $d$ satisfying the relation $c \sqsubseteq d$, then we can obtain $c^{\mathcal{J}}(\sigma) \neq c^{\mathcal{K}}(\tau)$ \cite{SWL2021}. This conclusion implies that if $c$ refines $d$, then $c$ is capable of distinguishing all non-isomorphic simplicial complex pairs that d can distinguish. In this sense, we consider $c$ to be at least as powerful as $d$.

Equipped with this background knowledge, we are now set up to prove the conclusions in the main article. To begin with, we introduce a slightly weak version of Theorem \ref{the:HWL}.

\begin{lemma}
\label{lemma:HWL}
HWL is at least as powerful as Weisfeiler-Lehman (WL) test in distinguishing non-isomorphic simplicial complexes.
\end{lemma}

\begin{proof}
\label{proof:HWL}
Let notion $a^{(t)}$ denote the coloring of the Weisfeiler-Lehman test with the coloring update rule  $ a^{(t+1)}(v)=\operatorname{Hash}$ $ \left(a^{(t)}(v), \multiset{a^{(t)}(u) | u \in N(v)} \right)$ and $c^{(t)}$ the coloring of HWL test for the same nodes in $\mathcal{K}$ using the rule $c^{(t+1)}=$ $\operatorname{Hash}$ $\left( c^{(t)}(\sigma), \multiset{c^{(t)}(\tau) | \tau \in \mathcal{N}(\sigma)} \right)$.  Note that the function $\operatorname{Hash}$ $ \left(a^{(t)}(v), \multiset{a^{(t)}(u) | u \in N(v)} \right)$ and $\operatorname{Hash}$  $ \left(\multiset{a^{(t)}(u) | u \in N(v)\cup \{v\} }\right)$ are equivalent.

We additionally introduce $b^{(t)}$ to represent the coloring of restricted HWL that utilizes information no larger than 1-simplex, i.e., the refine rule can be represented as $b^{(t+1)}(\sigma) = \operatorname{Hash}\left( b^{(t)}(\sigma), \multiset{b^{(t)}(\tau) | \tau \in \mathcal{N}_1(\sigma)} \right)$. Here, $\mathcal{N}_1(\sigma)$ denotes the set of neighbours of $\sigma$ in the bipartite graph $\mathcal{G}_1$.  It's trivial to show $c$ refines $b$ since it considers the additional colors from higher-order simplices. We now only have to prove that $b^{(2t)} \sqsubseteq a^{(t)}$ by induction.

The base case holds for all colors are the same at initialization.
For the induction step, suppose it satisfies $b^{(2t+2)}(v) = b^{(2t+2)}(u)$ for any two 0-simplices $v$ and $u$ in two arbitrary complexes. 
The inputs to the Hash function need to be the same as the SWL coloring is an injective mapping.
Hence, we can obtain $b^{(2t+1)}(v)=b^{(2t+1)}(u)$ and $\multiset{b^{(2t+1)}(\sigma)|\sigma \in \mathcal{N}_1(v) } = \multiset{b^{(2t+1)}(\tau)|\tau \in \mathcal{N}_1(u) }$.  By unwrapping the hash function, we can get
 \begin{equation}
    \begin{split}
       b^{(2t+1)}(\sigma) 
     & = \operatorname{Hash}\left(b^{(2t)}(\sigma), \multiset{b^{(2t)}(w)|w \in \mathcal{N}_1(\sigma)} \right)\\
     & =\operatorname{Hash}\left(b^{(2t)}(\sigma), \multiset{b^{(2t)}(w)|w \in N(v)\cup \{v\}}\right).
    \end{split}
\end{equation}

Similarly, we can obtain 
\begin{align}
    b^{(2t+1)}(\tau) = \operatorname{Hash}\left(b^{(2t)}(\tau),\multiset{b^{(2t)}(w)|w \in N(u)\cup \{u\}}\right).
\end{align}

By removing the same color $b^{(2t)}(\sigma)$ and $b^{(2t)}(\tau)$ from the input, we can further obtain that  $\multiset{b^{(2t)}(w)|w \in N(v)\cup\{v\}} = \multiset{b^{(2t)}(w)|w \in N(u)\cup\{u\}}$. According to the induction hypothesis $b^{(2t)} \sqsubseteq a^{(t)}$, it can be deduced that $\multiset{a^{(t)}(w)|w \in N(v)\cup\{v\}} = \multiset{a^{(t)}(w)|w \in N(u)\cup\{u\}}$. We can finally conclude that $a^{(t+1)}(v)=a^{(t+1)}(u)$, which implies $b^{(2t+2)} \sqsubseteq a^{(t+1)}$.
\end{proof}

\begin{figure}[!t]
\centering
\includegraphics[width=0.66\linewidth]{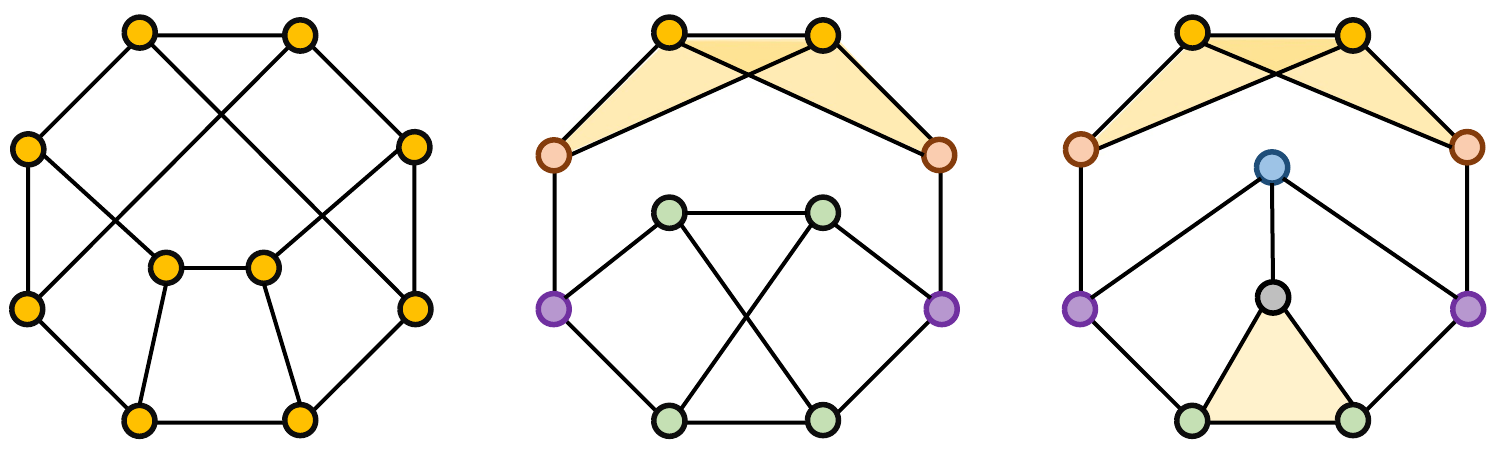}
\caption{Three non-isomorphic graphs that are indistinguishable by WL but distinguishable by HWL and SHWL with clique complex lifting. The incorporation of higher-order information in HWL and SHWL enables nodes to exhibit a more extensive diversity in color compared to the WL test, consequently yielding superior performance in both node-level and graph-level tasks.} 
\label{fig:HWL_toy}
\end{figure}

\begin{proof}[\textbf{Proof of Theorem \ref{the:HWL}}]
In accordance with Lemma \ref{lemma:HWL}, we only need to provide a pair of non-isomorphic graphs that WL fails to distinguish, while HWL can distinguish with the assistance of a clique complex lifting transition. Any two graphs in Figure \ref{fig:HWL_toy} constitute such a pair that satisfies these criteria.
\end{proof}

Moreover, we can simplify the coloring rule for simplices whose order is larger than 0 to nonvanishing linear functions, resulting in a simplified version of HWL, termed SHWL. 
The nonvanishing linear function implies that the polynomial coefficient will never be 0. 
We will demonstrate that the SHWL test remains more powerful than the WL test.

\begin{lemma}
\label{lemma:SHWL}
SHWL is at least as powerful as Weisfeiler-Lehman (WL) test in distinguishing non-isomorphic simplicial complexes.
\end{lemma}

\begin{proof} 
\label{proof:SHWL}
Let notion $a^{(t)}$ denote the coloring of WL test with the rule  $ a^{(t+1)}(v)=\operatorname{Hash}$ $ \left(a^{(t)}(v), \multiset{a^{(t)}(u) | u \in N(v)} \right)$ and $c^{(t)}$ the coloring of SHWL for the same nodes in $\mathcal{K}$ using $c^{(t+1)} =$ $ \operatorname{Hash}\left( c^{(t)}(\sigma), \multiset{c^{(t)}(\tau) | \tau \in \mathcal{N}(\sigma)} \right)$.

We additionally introduce $b^{(t)}$ to represent the coloring of restricted SHWL that utilizes information no larger than 1-simplex, i.e. the refine rule can be represented as $b^{(t+1)}(\sigma) = \operatorname{Hash}\left( b^{(t)}(\sigma), \multiset{b^{(t)}(\tau) | \tau \in \mathcal{N}_1(\sigma)} \right)$. Here, $\mathcal{N}_1(\sigma)$ denotes the set of neighbours of $\sigma$ in the bipartite graph $\mathcal{G}_1$.  It's trivial to show $c \sqsubseteq b$ since it considers the additional colors from higher-order simplices. We now only have to prove that $b^{(2t)} \sqsubseteq a^{(t)}$ by induction.

The base case holds for all colors are the same at initialization.
 For the induction step, suppose it satisfies $b^{(2t+2)}(v) = b^{(2t+2)}(u)$ for any two 0-simplices $v$ and $u$ in two arbitrary complexes. Then we can obtain $b^{(2t+1)}(v)=b^{(2t+1)}(u)$ and $\multiset{b^{(2t+1)}(\sigma)|\sigma \in \mathcal{N}_1(v) } = \multiset{b^{(2t+1)}(\tau)|\tau \in \mathcal{N}_1(u) }$, since these are the arguments of the hash function.  By unwrapping the coloring rule, we can obtain
 \begin{equation}
    \begin{split}
       b&^{(2t+1)}(\sigma) 
      = \operatorname{Linear}\left(b^{(2t)}(\sigma), \multiset{b^{(2t)}(w)|w \in \mathcal{N}_1(\sigma)} \right)\\
     & =\operatorname{Linear}\left(b^{(2t)}(\sigma), \multiset{b^{(2t)}(v), b^{(2t)}(w)|w \in N(v)}\right) \\
     & = \operatorname{Linear}\left( \multiset{b^{(2t)}(w)|w \in N(v)\cup \{v, \sigma\} }\right).
    \end{split}
\end{equation}

Similarly, we can obtain 
\begin{align}
    b^{(2t+1)}(\tau) = \operatorname{Linear}\left( \multiset{b^{(2t)}(w)|w \in N(u)\cup \{u, \tau\} }\right).
\end{align}

Since the coloring $b^{(2t)}(\sigma)$ and $b^{(2t)}(\tau)$ are the same, we can further obtain  $\operatorname{Linear}\left( \multiset{b^{(2t)}(w)|w \in N(v)\cup \{v} \}\right)=\operatorname{Linear}\left( \multiset{b^{(2t)}(w)|w \in N(u)\cup \{u\} }\right)$ by removing the same color from the linear function. Because the linear function is nonvanishing, we can finally conclude that $a^{(t+1)}(v)=a^{(t+1)}(u)$, which implies $b^{(2t+2)} \sqsubseteq a^{(t+1)}$.
\end{proof}

\begin{proof} [\textbf{Proof of Theorem \ref{the:SHWL}}]
According to Lemma \ref{lemma:SHWL}, we only have to offer a pair of non-isomorphic graphs that WL fails to distinguish, while SHWL can distinguish with the help of clique complex lifting transition. It might as well to assume the nonvanishing linear function to be a summation function, and such a non-isomorphic graph pair is demonstrated in Figure \ref{fig:HWL_toy}.
\end{proof}

\clearpage
\section{Relation to other GCN models}
\label{appendix:relation_to_GCNs}

If we consider only the information transfer process based on the edge (1-simplex) in the bipartite graph $\mathcal{G}_1$, then the higher-order FP adjacency matrix can be represented as:
\begin{equation}
    \tilde{\mathcal{A}_1} = \frac{1}{2}D^{-1/2}\mathcal{H}_1 \mathcal{H}_1^\top D^{-1/2}=\frac{1}{2}\left(D^{-1/2}AD^{-1/2}+I\right).
    \label{equ:A1}
\end{equation}
Here, we omit the subscript for $D_{1,v}$ to unload the notation, for it retains the same meaning as in pairwise graphs without causing ambiguity.
It can be drawn that the definition of FP adjacency matrices preserves the topological feature of the original graph. 
Additionally, there is no need to add self-loops since they are already implicitly included in the FP adjacency matrices.
Therefore, HiGCN can be viewed as an extension of vanilla spatial GNNs within the higher-order domain.


In particular, GPRGNN \cite{GPRGNN}, one of the state-of-the-art spectral GNNs related to HiGCN, can be considered a special case of HiGCN if we only account for information interactions between $0$-simplices and $1$-simplices, i.e., requiring parameter $P=1$. Moreover, GCN \cite{GCN} also constitutes a specific instance of our model if we employ a fixed low-pass filter and fix $P=1$.

On the other hand, HiGCN demonstrates enhanced flexibility over certain Hodge Laplacian-based simplicial GCNs, transcending the constraints imposed by information exchange exclusively through boundary operators.
Specifically, HiGCN generalizes simplicial convolution operations on simplicial complexes, encompassing SNN \cite{SNN2020} and SCNN \cite{SCNN2022}, if we apply fixed low-pass filters and require $P=1$ for node-level tasks.



\section{Random walk model}
\label{appendix:RW}

Graph neural networks can also be interpreted as a modified random walk of information over graph structures \cite{APPNP}. 
The main idea of random walks is to traverse a graph starting from a single node or a set of nodes and get sequences of locations \cite{HoRW}. 

In unbiased random walk models, suppose a random walker is located at node $i$ and he will wander to a neighbouring vertex $j$ with probability $A_{ij}/k_j$, reflecting an equal selection between $k_j$ neighbours. In mathematical, this process can be formulated as
\begin{equation}
    \pi_i(t) = \sum_j \frac{A_{ij}}{k_j} \pi_j(t-1).
\end{equation}
Here, $k_j$ stands for the degree of node $j$.
It can be expressed equivalently as
\begin{equation}
    \pi(t)=AD^{-1}\pi(t-1),
    \label{equ:RW}
\end{equation}
where $\pi(t) = \left(\pi_{1}(t),\cdots,\pi_{n}(t)\right)^\top$ and $D = \operatorname{diag}\left(k_1,\cdots,k_n\right)$.

Multiplying $D^{-1/2}$ from left sides of the equation (\ref{equ:RW}) simultaneously, we can obtain
\begin{equation}
    D^{-1/2}\pi(t)= \left[D^{-1/2}AD^{-1/2}\right] D^{-1/2}\pi(t-1).
    \label{equ:reducedA}
\end{equation}
Here, $D^{-1/2}AD^{-1/2}$ is a symmetric matrix and is referred to as a reduced adjacency matrix. Based on the reduced adjacency matrix, we can obtain the normalized graph Laplacian
\begin{equation}
    L = I - D^{-1/2}AD^{-1/2}.
\end{equation}

\section{Equivariance and invariance}
\label{appendix:symmetry}

Equivariance and invariance are fundamental concepts in understanding GNNs and their behavior when processing graph-structured data.
Given a pairwise graph $\mathcal{G}$ characterized by adjacency matrix $A$ and feature matrix $X$, a function $f$ is (node) permutation equivariant if $\mathbf{P}f(A, X)=f(\mathbf{P}A\mathbf{P}^\top, \mathbf{P}X)$ for any permutation operation $\mathbf{P}$. 
GNNs adhere to this equivariance property, ensuring consistent computation of functions irrespective of node permutations.

Generalizing pairwise GNNs, HiGCN demonstrates equivariance with respect to relabeling of simplices, which allows it to exploit symmetries in SCs. 
In accordance with recent endeavors in simplicial neural networks to examine various models in terms of equivariance \cite{SWL2021, SCNN2022}, our objective herein is to elucidate the fundamental equivariance properties intrinsic to the HiGCN model.

To formalize this property, we introduce the concept of simplex permutation equivariance and demonstrate that HiGCN fulfills this definition. 

Consider a simplicial complex $\mathcal{K}$ with a maximum simplex order of $\ell$ and a sequence $\mathbf{H}=\left(\mathcal{H}_1,\cdots,\mathcal{H}_\ell\right)$ of higher-order incidence matrices. 
Let $\mathbf{P} = \left(P, P_1,\cdots,P_\ell\right)$ represent a sequence of simplicial permutation matrices with $P\in \mathbb{R}^{n \times n}$ and $P_i\in \mathbb{R}^{n_i \times n_i}$. 
Denote the permutation features as $PX$ and the sequence of permutation higher-order incidence matrices $\left(P \mathcal{H}_1 P_1, \cdots, P\mathcal{H}_\ell P_\ell\right)$ as $\mathbf{PHP}^\top$.

\begin{definition}
A function $f$ is simplex permutation equivariant if $f\left(PX, \mathbf{PHP}^\top\right) = Pf\left(X, \mathbf{H}\right)$ for any sequence of permutation operators $\mathbf{P}$.
\end{definition}

\begin{theorem}
\label{the:equivariance}
The HiGCN model is simplex permutation equivariant.
\end{theorem}

\begin{proof}
After a sequence of permutations $\mathbf{P}$, the higher-order incidence matrix $\mathcal{H}_p$ is transformed into $P\mathcal{H}_p P_p^\top$ and  $D_{p,v}$ becomes $PD_{p,v}P^\top$. Consequently,  the permuted FP adjacency matrix $\tilde{\mathcal{A}_p}$ is given by 
\begin{equation}
\begin{split}
    \tilde{\mathcal{A}_p}^\prime
     = & \frac{1}{p+1} {D_{p,v}^\prime}^{-1/2}\mathcal{H}_p^\prime{\mathcal{H}_p^\prime}^\top {D_{p,v}^\prime}^{-1/2}\\
     = & \frac{1}{p+1} \left(P D_{p,v}^{-1/2} P^\top \right)
    \left(P\mathcal{H}_p P_p^\top \right)
    \left(P_p {\mathcal{H}_p}^\top P^\top \right)\cdot\\
    &\left(P  D_{p,v}^{-1/2} P^\top \right)\\
     = & \frac{1}{p+1} P D_{p,v}^{-1/2} \mathcal{H}_p {\mathcal{H}_p}^\top D_{p,v}^{-1/2} P^\top \\ 
     = & P \tilde{\mathcal{A}_p} P^\top .
\end{split}
\end{equation}

Therefore, we can deduce $\tilde{\mathcal{A}_p^\prime}^k = P \tilde{\mathcal{A}_p^k} P^\top$. Subsequently, we can express the permuted convolution output as
\begin{equation}
    Y^\prime =   \mathop{\Big|\!\Big|}\limits_{p} \left(\sum_{k=0}^K \gamma_{p,k} P \tilde{\mathcal{A}_p^k} P^\top 
 PX \Theta_p \right) W = P Y.
\end{equation}

Moreover, as the employed non-linear activation function and readout function exhibit equivariance, the HiGCN model is simplex permutation equivariant.
\end{proof}

Comprehending equivariance and invariance in GNNs is of paramount importance for devising effective models and training methodologies. By taking these properties into account, GNN architectures can be developed to capitalize on equivariance in order to capture and propagate local structural information, while simultaneously ensuring invariance to concentrate on graph-level or task-specific features.

\section{Experimental details for quantification of higher-order interaction strengths}
\label{appendix:S_p}


\subsection{Construction of 1k null model}
\label{appendix:1k}

The null model \cite{nullmodel,zeng2023hyper} refers to a class of random networks endowing with some of the same properties as an actual network, and they are often employed to quantify the extent of some special properties.
The 1k null model \cite{YRM2023} is the network that has the same node number and edge number as the origin network.
The construction of the 1k null model with a changeable triangle number is divided into two parts: (1) node-picking and (2) relinking. 

In the node-picking process, we traverse all nodes in the graph and find node A with more than two neighbors. 
Based on node A, we pick two neighbor nodes node B and node C which are not connected.
We then pick neighbor nodes D and E of node B and node C, which should not be connected with node A. 
Notice that nodes D and E can not connect with each other, so the whole chain of D-B-A-C-E contains no triangle.

In the relinking step, we break the original edges $[B, D]$ and $[C, E]$, and then add edges $[B, C]$ and $[D, E]$ (red edges in Figure \ref{fig:1knull}). 
In this way, a new triangle $[A, B, C]$ (gray shade in  Figure \ref{fig:1knull}) is added with the number of nodes and edges stable.

\begin{figure}[h]
\centering
\includegraphics[width=0.4\linewidth]{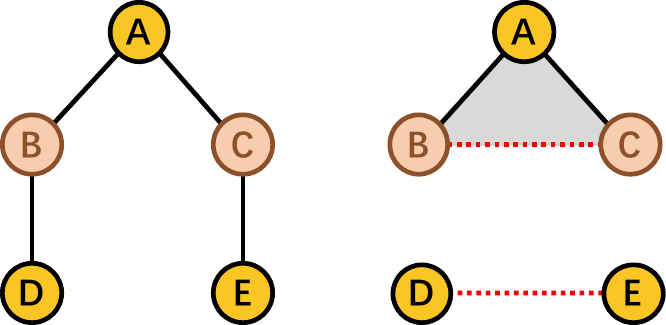}
\caption{The modification process of the 1k null model. The left and right figures show the partial structures of the networks before and after modification, respectively.}
\label{fig:1knull}
\end{figure}





\subsection{More experimental results}
\label{appendix: more_exp}

HiGCN achieves the optimal results in node classification tasks except for Texas, see Table \ref{tab:node_classify}.
It has also been noticed in Figure \ref{fig:stack5}(\textbf{d}) that Texas's higher-order interactions are pretty weak compared to its lower-order ones, along with the least triangles over all datasets (see Table \ref{tab:node_statistics}).
Therefore, few higher-order structures may lead to weak higher-order influence, along with degraded node classification performance of HiGCN. 
To verify this conjecture, we modify the number of higher-order structures by adjusting the connectivity of edges in the network while maintaining the degree distribution as in 1k null models.

Table \ref{tab:higherorder_classify_Texas} shows the node classification results under different relative triangle densities from 0 (original network) to $+ 50 \%$ in units of $10 \%$.
An increasing accuracy rank is observed between HiGCN and the benchmark with the rise of the triangles' density $\rho_2$.
In particular, HiGCN performs best at higher relative triangle densities $\rho_2 = 20 \%$, $40 \%$, and $50 \%$.

\begin{table*}[!ht] 
\centering
\begin{tabular}{cccccccccccc}
\toprule
$\rho_2$    &0   &10\%   &20\%   &30\%  &40\%  &50\%\\
\midrule
MLP
&91.45\spm{1.14}  &91.02\spm{0.98}  &91.90\spm{0.95}  &91.18\spm{1.11}  &91.74\spm{0.92}  &91.05\spm{0.95}\\
GCN      
&75.16\spm{0.96}  &64.36\spm{1.83}  &64.89\spm{2.16}  &64.07\spm{1.93}  &62.92\spm{2.34}  &64.75\spm{2.16}\\
GAT      
&78.87\spm{0.86}  &79.97\spm{1.03}  &78.89\spm{1.07}  &78.20\spm{1.20}  &77.28\spm{1.30}  &77.93\spm{1.42}\\
ChebNet  
&86.08\spm{0.96}  &82.10\spm{1.52}  &83.08\spm{1.09}  &79.21\spm{1.55}  &81.34\spm{1.48}  &81.41\spm{1.34}\\
BernNet  
&\textbf{93.12\spm{0.65}}  &\textbf{92.10\spm{0.95}}  &92.89\spm{0.92}  &\textbf{92.33\spm{1.08}}  &91.54\spm{1.02}  &91.57\spm{1.21}\\
GGCN  
&85.81\spm{1.72}  &85.95\spm{1.42}  &85.51\spm{1.67}  &83.84\spm{1.70}  &91.15\spm{1.02}  &83.95\spm{1.73}\\
APPNP  
&90.98\spm{1.64}  &89.87\spm{1.01}  &89.31\spm{1.05}  &89.08\spm{0.98}  &90.39\spm{1.10}  &89.57\spm{1.10}\\
GPRGNN 
&92.95\spm{1.30}  &86.10\spm{2.76}  &88.16\spm{1.13}  &83.05\spm{2.05}  &84.69\spm{1.77}  &83.54\spm{2.72}\\
\midrule
\textbf{HiGCN}
&92.15\spm{0.73}  &91.70\spm{1.06}  &\textbf{93.11\spm{0.87}}  &91.64\spm{1.14}  &\textbf{91.93\spm{0.84}} &\textbf{92.24\spm{1.41}}\\
\midrule
Rank
&3  &2  &\textbf{1}  &2  &\textbf{1} &\textbf{1}\\
\bottomrule
\end{tabular}
\caption{Node classification results on Texas with changeable higher-order densities: mean accuracy $(\%)\pm 95\%$ confidence interval. Bold values indicate the best result.}
\label{tab:higherorder_classify_Texas}
\end{table*}

The process of constructing a 1k null model can also be considered as adding specific noise to the network, and it can be found from Table \ref{tab:higherorder_classify_Texas} that our model excels in handling noise containing higher-order information.





\section{Experiments on large-scale datasets and computational analysis}
\label{appendix: large}

We further conduct experiments on three larger datasets: two homogeneous graphs (Ogbn-arxiv \cite{data:ogbn_arxiv} and Genius \cite{data:genius}) and one heterogeneous graph (Penn94 \cite{data:penn94}). 
We compare our method with MLP, S2V (simplex2vec) \cite{Simplex2vec}, as well as several state-of-art GNN models including GCN \cite{GCN}, GAT \cite{GAT2018}, ChebNet \cite{ChebNet}, GPRGNN \cite{GPRGNN},  LGGNN, DRGCN \cite{DRGCN}, GloGNN++ \cite{GloGNN}, and ACM-GCN++ \cite{ACM-GCN}.
The results presented in Table \ref{tab:bigdataset} demonstrate that our HiGCN model outperforms these comparative methods in the node classification task while maintaining a reasonable computational cost.


\begin{table}[!ht]
\centering
\resizebox{0.47\textwidth}{!}{
\begin{tabular}{cccc}
\toprule
Method & Penn94          & Ogbn-arxiv      & Genius \\
\midrule
MLP     & 73.61\spm{0.40} & 69.88\spm{0.27} & 86.68\spm{0.09} \\
S2V     & 80.60\spm{0.22} & 68.87\spm{0.12} & 83.21\spm{0.34} \\
GCN     & 82.47\spm{0.27} & 71.49\spm{0.30} & 87.42\spm{0.37} \\
GAT     & 81.53\spm{0.55} & 71.59\spm{0.38} & 55.80\spm{0.87} \\
ChebNet & 81.21\spm{0.32} & 71.12\spm{0.22} & 85.69\spm{0.75} \\
GPRGNN  & 81.38\spm{0.16} & 71.78\spm{0.18} & 90.09\spm{0.31} \\
LGGNN & N/A & 75.70\spm{0.18} & N/A\\
DRGCN & N/A & \s{76.11}\spm{0.09} & N/A\\
GloGNN++  & \s{85.74}\spm{0.42} &  N/A & 90.91\spm{0.13}\\
ACM-GCN++ & \f{86.08}\spm{0.43} &  N/A & 91.40\spm{0.07}\\
\midrule
1-HiGCN & 83.01\spm{0.44} & 73.86\spm{0.25} & 89.98\spm{0.73}\\
2-HiGCN & 85.66\spm{0.59} & 76.00\spm{0.41} & \thi{91.53}\spm{0.36}\\
3-HiGCN & \thi{85.68}\spm{0.67} & \f{76.41}\spm{0.53} & \s{91.60}\spm{0.61}\\
4-HiGCN & 84.98\spm{0.32} & \thi{76.01}\spm{0.38} & \f{91.66}\spm{0.43}\\
\bottomrule
\end{tabular}}
\caption{Node classification results on large-scale networks: mean accuracy $(\%)\pm 95\%$ confidence interval. The best results are in bold, while the second-best ones are
underlined.}
\label{tab:bigdataset}
\end{table}

In addition, we examine the computational complexity of HiGCN compared to other node classification baselines and report the average training time per epoch and average total running time in Table \ref{tab: runTime}.

\begin{table*}[!ht]
\centering
\begin{tabular}{ccccccccc}
\toprule
Dataset    & 2-HiGCN     & 3-HiGCN   & 4-HiGCN    & APPNP        & BERNNET     & GPRGNN       &CHEBNET     &GGCN \\
\midrule
Cora        & 2.8 / 0.6  & 3.0 / 0.7     & 3.3 / 0.9    &3.6 / 1.2   & 11.6 / 3.1  &4.3 / 0.9        &4.6/2.2  &5.2/2.4\\
Citeseer    & 3.0 / 0.7  & 3.2 / 0.7     & 3.4 / 0.9    &3.7 / 1.3   & 11.8 / 3.4  &4.5 / 1.0        &5.4/2.4  &2.4/2.7\\
Pubmed      & 4.8 / 1.0  & 5.7 / 1.2     & 6.3 / 1.3    &3.9 / 2.0   & 11.1 / 4.9  &4.5 / 1.8        &6.0/3.0  &7.2/7.3\\
Computers   & 6.2 / 1.3  & 6.5 / 1.4     & 7.0 / 1.6    &6.0 / 2.5   & 29.3 / 8.6  &6.5 / 1.6    &21.1/12.7  &18.3/15.1\\
Photo       & 5.3 / 1.1  & 5.6 / 1.2     & 6.3 / 1.3    &5.8 / 2.8   & 15.3 / 6.2  &4.5 / 1.3    &19.8/11.9  &19.0/10.3\\
\midrule
Chameleon   & 4.4 / 1.3  & 4.8 / 1.5     & 5.3 / 1.6    &3.9 / 0.8   & 11.0 / 2.8  &4.4 / 1.0   &5.0/2.3  &4.8/5.2\\
Actor       & 2.7 / 0.8  & 3.1 / 0.9     & 3.5 / 1.1    &3.8 / 0.8   & 10.9 / 3.5  &4.3 / 0.9   &4.2/1.8  &4.7/4.9\\
Squirrel    & 8.5 / 2.6  & 8.9 / 2.8     & 9.6 / 3.0    &4.3 / 0.9   & 15.7 / 4.9  &4.3 / 2.1   &6.3/3.5  &11.4/14.7\\
Texas       & 2.4 / 0.7  & 2.6 / 0.8     & 2.8 / 0.8    &3.8 / 0.8   & 11.3 / 2.4  &4.3 / 1.0   &2.5/0.9  &2.1/1.5\\
Wisconsin   & 2.5 / 0.7  & 3.0 / 0.9     & 3.2 / 1.0    &3.9 / 0.8   & 11.0 / 2.6  &4.4 / 0.9   &2.6/1.0  &2.6/1.9\\
\bottomrule
\end{tabular}
\caption{Efficiency on node classification experiments: Average running time per epoch(ms)/ average total running time(s).}
\label{tab: runTime}
\end{table*}

Furthermore, when the targeted graph is not in the form of SCs, one should also consider the one-time preprocessing procedure for graph lifting.
Specifically, the number of $p$-simplices in a graph with $n$ nodes and $m$ edges is upper-bounded by $\mathcal{O}(n^{p-1})$, and they can be enumerated in $\mathcal{O}( a\left(\mathcal{G}\right)^{p-3} m)$ time \cite{chiba1985arboricity}, where $a\left(\mathcal{G}\right)$ is the arboricity of the graph $\mathcal{G}$, a measure of graph sparsity.
Since arboricity is demonstrated to be at most $\mathcal{O}(m^{1/2})$ and $m \leq n^2$, all $p$-simplices can thus be listed in $\mathcal{O}\left( n^{p-3} m \right)$.
Besides, the complexity of finding $2$-simplex is estimated to be $\mathcal{O}(\left\langle k \right\rangle m )$ with the Bron–Kerbosch algorithm \cite{find_cliques1973}, where $\left \langle k \right \rangle$ denotes the average node degree, typically a small value for empirical networks.

\section{Datasets}
\label{appendix: datasets}

\paragraph{Node classification.}
Table \ref{tab:node_statistics} provides the statistics of all the datasets used in the node classification experiments.
The homophily level of $\mathcal{G}$ is measured by $Homophily(\mathcal{G})=\frac{\left|\left\{(u, v):(u, v) \in \mathcal{E} \wedge y_v=y_u\right\}\right|}{n_1}$ \cite{homo2020}, where $y_v$ is the label of nodes $v$ and $n_1 = |\mathcal{E}|$ denotes the number of edges.

We additionally scale HiGCN to three large-scale datasets: Penn94, Ogbn-arxiv, and Genius.
\textbf{Penn94} \cite{data:penn94} is a social network extracted from Facebook where nodes represent students and edges denote the communication relationship of students.
\textbf{Ogbn-arxiv} \cite{data:ogbn_arxiv} is a paper citation network of arxiv papers extracted from the Microsoft Academic Graph (MAG) where the nodes denote the papers and the edges denote the citation relationship.
\textbf{Genius} \cite{data:genius} is a social network extracted from a website for crowdsourced annotations of song lyrics named ``genius.com".
The nodes and edges of it represent users and the following relationship of users.
%

\begin{table*}[!htbp]
\centering
\begin{tabular}{cccccccc}
\toprule
Network   & Classes & Features & Homophily & $n $  & $n_1$  & $n_2$  &$\left \langle k \right \rangle$ \\
\midrule
Cora      & 7       & 1433     & 0.8099    & 2708  & 5278   & 1630      &3.90    \\
Citeseer  & 6       & 3703     & 0.7355    & 3327  & 4552   & 1167      &2.74\\
PubMed    & 5       & 500      & 0.8023    & 19717 & 44324  & 12520     &4.50\\
Computers & 10      & 767      & 0.7772    & 13752 & 245861 & 1527469   &35.76\\
Photo     & 8       & 745      & 0.8272    & 7650  & 119081 & 717400    &31.13\\
\midrule
Chameleon & 5       & 2325     & 0.2305    & 2277  & 31421  & 343066    &27.60\\
Actor     & 5       & 2089     & 0.2187    & 7600  & 26659  & 7121      &7.02\\
Squirrel  & 5       & 932      & 0.2224    & 5201  & 198493 & 9595609   &76.33\\
Texas     & 5       & 1703     & 0.0871    & 183   & 279    & 67        &3.05\\
Wisconsin & 5       & 1703     & 0.1921    & 251   & 450    & 118       &3.56\\
\midrule
Penn94    & 2       & 5        & 0.470     & 41554  & 1362229  & 7207796    &65.6\\
Ogbn-arxiv& 40       & 128        & 0.655     & 169343 & 1157799  & 2233524      &13.7\\
Genius  & 2       & 12        & 0.618     & 421961 & 984979  & 1968352   &4.7\\
\bottomrule
\end{tabular}
\caption{Statistics of node classification datasets.}
\label{tab:node_statistics}
\end{table*}

\paragraph{Graph classification.}
In this study, we investigate various datasets from diverse domains to evaluate the performance of the proposed framework. The datasets are divided into two main categories: bioinformatics datasets and social network datasets.

The bioinformatics datasets comprise MUTAG, PTC, and PROTEINS.
\textbf{MUTAG} \cite{data:MUTAG} comprises  188 mutagenic aromatic and heteroaromatic nitro compounds with seven discrete labels. The goal is to identify mutagenic molecular compounds for potential drug development.
\textbf{PTC} \cite{data:PTC} involves classifying graph-structured compounds based on their carcinogenic properties in rodents, containing 344 chemical compounds with 19 discrete labels that indicate carcinogenicity for male and female rats.
\textbf{PROTEINS} \cite{data:proteins2005} aims to classify proteins into enzyme and non-enzyme structures. Nodes represent secondary structure elements (SSEs) connected by edges if they are neighbors in the amino-acid sequence or in 3D space. There are three discrete labels: sheet, helix, and turn structures.
As for social network datasets, \textbf{IMDB-B} \& \textbf{IMDB-M} \cite{data:IMDB-REDDIT} are movie collaboration datasets consisting of actors' and actresses' ego-networks from IMDB.  Nodes represent actors/actresses, and edges connect them if they appear in the same movie. Each graph originates from a specific movie genre and the task is to classify their genre. The primary difference between the two datasets lies in the number of categories: IMDB-B has two (Action and Romance), while IMDB-M includes three (Comedy, Romance, and Sci-Fi).
Table \ref{tab:graph_datasets} provides the statistics of the graph classification datasets employed in this study.

\begin{table*}[!htbp]
\centering
\begin{tabular}{ccccccc}
\toprule
Dataset   
& Graphs
& Classes   
& $\left \langle n \right \rangle$ 
& $\left \langle n_1 \right \rangle$ 
& $\left \langle n_2 \right \rangle$ 
& $\left \langle k \right \rangle$ \\  
\midrule
MUTAG       & 188   & 2     & 17.93     & 19.79     & 0.00     &2.21\\
PTC         & 344   & 2     & 25.56     & 25.96     & 0.04  &2.03\\
PROTEINS    & 1113  & 2     & 39.06     & 72.82     & 27.40  & 3.73        \\
IMDB-B      & 1000  & 2     & 19.77     & 96.53     & 391.99 & 9.77\\
IMDB-M      & 1500  & 3     & 13.00     & 65.94     & 305.90 & 10.14\\
\bottomrule
\end{tabular}
\caption{Statistics of graph classification datasets.}
\label{tab:graph_datasets}
\end{table*}

\paragraph{Simplicial data imputation.}
We extract three coauthorship complexes from DBLP \cite{data:DBLP-Benson2018}, History and Geology \cite{data:MAG-His-Geo}. In these coauthorship complexes, a paper with $p+1$ authors is represented by a $p$-simplex, and the $p$-simplicial signal corresponds to the number of collaborative publications among these authors. 
To mitigate the potential noise introduced by incidental collaborations, we consider coauthor groups with more than two  collaborative publications as simplices. 
Following the standard pipeline for this task \cite{SNN2020}, missing values are artificially introduced by replacing a portion of the signals with a constant.
Table \ref{tab:CCs_statistics} provides the statistics of the coauthorship complexes that are employed.

\begin{table*}[!htbp]
\centering
\begin{tabular}{cccccccc}
\toprule
SCs   & $n$   & $n_1$ & $n_2$   & $n_3$ & $n_4$ &$n_5$ \\  
\midrule
DBLP    &41910  &23334  &10951  &6279   &4684   &3756 \\
Geology &52259  &52373  &50051  &62869  &98846  &159879\\
History &29354  &5554   &9334   &24428  &65770  &166285\\
\bottomrule
\end{tabular}
\caption{Statistics of coauthorship complexes.}
\label{tab:CCs_statistics}
\end{table*}

\section{Supplementary experimental settings}
\label{appendix: exp_setting}

The well-known small-world effect \cite{SmallWorld1967} in network science informs that even large-scale networks possess a finite diameter, implying that the neighbor scope under consideration (denoted by $K$) need not be excessively large, typically around 6. In all our experiments, we set $K=10$, a value considerably larger than the average network diameter.
%

\paragraph{Node classification.}
We train the proposed HiGCN model with the learning rate $lr \in \{0.01, 0.05, 0.1, 0.2, 0.3\}$ and the weight decay $wd \in \{0.0, 0.0001, 0.001, 0.005, 0.1\}$.
We leverage Adam as the model optimizer and establish the maximum number of epochs at 1000 for all datasets. 
We utilize 2 linear layers with 32 hidden units for the NN component. 
The activation function is log Softmax, and the loss function is negative log Likelyhood loss (NLL loss). 
We employ RayTune to randomly select hyperparameters with the goal of optimizing the accuracy score on the validation set.

\paragraph{Graph classification.}
We set the parameter $P=2$ in this experiment, implying that the highest order of the simplices under consideration is $2$-simplices, i.e., each graph is lifted to a clique $2$-complex, where $0$-simplices are initialized with the original node signals as prescribed in \cite{GIN2019}. 
Our training process starts from an initial learning rate, which decayed after a fixed amount of epochs.
We implement readout operations by conducting averaging or summation depending on the dataset to calculate feature vectors derived from the hidden states. Although we do not deploy any sophisticated simplicial complex pooling operator within this study, we anticipate that this concept presents a compelling trajectory for future research endeavors.
We run a grid search to tune batch size among $\{32, 64\}$, hidden dimension among $\{32, 64\}$, dropout rate among $\{0.0, 0.3\}$, initial learning rate among $\{0.0005, 0.001, 0.005, 0.01\}$.  
We follow the same evaluation protocol of \cite{GIN2019}, conducting a 10-fold cross-validation procedure and reporting the maximum average validation accuracy across folds.

\paragraph{Simplicial data imputation.} 
We use a two-layer (MLP) with 32 hidden units for the NN component.
We adopt the $\ell_1$ norm to train the NNs over known signals for 500 iterations using the Adam optimizer with a learning rate  $ lr \in \{0.001, 0.005,  0.01, 0.05\}$, $\alpha \in \{0.1,0.3,0.5,0.7,0.9\}$ and weight decay $wd=0$. Each experiment is performed for 10 different random weight initializations.
We run a grid search to select hyperparameters with the goal of optimizing the Kendall correlation for the entire signal.

\paragraph{Baselines.}
For the baseline methods BernNet \cite{BernNet}, GGCN \cite{GGCN}, APPNP \cite{APPNP}, GPRGNN \cite{GPRGNN}, GIN \cite{GIN2019}, S2V \cite{Simplex2vec}, SNN \cite{SNN2020}, SGAT \cite{SGAT}, SGATEF \cite{SGAT} and MPSN \cite{SWL2021}, we directly use the code provided by the authors. 
For the baseline methods LGGNN, DRGCN \cite{DRGCN}, GloGNN++ \cite{GloGNN} and ACM-GCN++ \cite{ACM-GCN}, we directly use the results provided by the authors.
For the remaining methods, we use the Pytorch Geometric library \cite{graphwithPyTorch2019} for implementations. 
We adopt the best hyperparameters provided by the authors if available; otherwise, hyperparameters are searched within the same hyperparameter space as our model. The hyperparameters adopted can be found in our code.

\paragraph{Computing infrastructure.} 
We utilize NetworkX, Pytorch, and Pytorch Geometric for model construction. All experiments are conducted on two Nvidia GeForce RTX 3060 GPUs with 64 GB memory.

\end{document}